\documentclass{article}

\usepackage[dblblindworkshop, final]{neurips_2026}

\usepackage[utf8]{inputenc}
\usepackage[T1]{fontenc}
\usepackage{url}

\providecommand{\doi}{}
\renewcommand{\doi}[1]{\href{https://doi.org/#1}{doi:#1}}
\usepackage{booktabs}
\usepackage{amsfonts}
\usepackage{amssymb}
\usepackage{amsmath}
\usepackage{nicefrac}
\usepackage{microtype}
\usepackage{xcolor}
\usepackage{graphicx}
\usepackage{float}
\usepackage{placeins}
\usepackage{enumitem}
\usepackage{amsthm}
\usepackage{soul}
\usepackage{hyperref}

\theoremstyle{plain}
\newtheorem{proposition}{Proposition}

\theoremstyle{definition}
\newtheorem{definition}{Definition}

\theoremstyle{remark}
\newtheorem{remark}{Remark}

\providecommand{\SSd}[2]{\csname SSdisc#1u#2\endcsname}
\providecommand{\SSarg}[3]{\csname SSarg#1#2u#3\endcsname}
\providecommand{\SSrng}[3]{\csname SSrng#1#2u#3\endcsname}
\providecommand{\SSlgt}[3]{\csname SSlgt#1#2u#3\endcsname}
\expandafter\def\csname SSdisc256u05\endcsname{9.5}
\expandafter\def\csname SSdisc256u09\endcsname{1.9}
\expandafter\def\csname SSdisc256u099\endcsname{0.9}
\expandafter\def\csname SSdisc1024u05\endcsname{16.1}
\expandafter\def\csname SSdisc1024u09\endcsname{6.5}
\expandafter\def\csname SSdisc1024u099\endcsname{1.9}
\expandafter\def\csname SSdisc4096u05\endcsname{37.1}
\expandafter\def\csname SSdisc4096u09\endcsname{11.3}
\expandafter\def\csname SSdisc4096u099\endcsname{0.8}
\expandafter\def\csname SSdisc16384u05\endcsname{86.0}
\expandafter\def\csname SSdisc16384u09\endcsname{28.7}
\expandafter\def\csname SSdisc16384u099\endcsname{2.0}
\expandafter\def\csname SSdisc65536u05\endcsname{186.4}
\expandafter\def\csname SSdisc65536u09\endcsname{71.4}
\expandafter\def\csname SSdisc65536u099\endcsname{5.6}

\expandafter\def\csname SSargEnergyHalf256u05\endcsname{1.5}
\expandafter\def\csname SSrngEnergyHalf256u05\endcsname{$9.5$}
\expandafter\def\csname SSlgtEnergyHalf256u05\endcsname{$0.5$}
\expandafter\def\csname SSargEnergyOne256u05\endcsname{1.5}
\expandafter\def\csname SSrngEnergyOne256u05\endcsname{$1.2\times 10^{3}$}
\expandafter\def\csname SSlgtEnergyOne256u05\endcsname{$0.4$}
\expandafter\def\csname SSargEnergyOneHalf256u05\endcsname{1.8}
\expandafter\def\csname SSrngEnergyOneHalf256u05\endcsname{$5.2\times 10^{5}$}
\expandafter\def\csname SSlgtEnergyOneHalf256u05\endcsname{$0.0$}
\expandafter\def\csname SSargSliced256u05\endcsname{3.5}
\expandafter\def\csname SSrngSliced256u05\endcsname{$8.6\times 10^{3}$}
\expandafter\def\csname SSlgtSliced256u05\endcsname{$-0.3$}
\expandafter\def\csname SSargHill256u05\endcsname{1.5}
\expandafter\def\csname SSrngHill256u05\endcsname{$2.9$}
\expandafter\def\csname SSlgtHill256u05\endcsname{$1.3$}
\expandafter\def\csname SSargEnergyHalf256u09\endcsname{1.5}
\expandafter\def\csname SSrngEnergyHalf256u09\endcsname{$1.9$}
\expandafter\def\csname SSlgtEnergyHalf256u09\endcsname{$0.1$}
\expandafter\def\csname SSargEnergyOne256u09\endcsname{1.5}
\expandafter\def\csname SSrngEnergyOne256u09\endcsname{$82.0$}
\expandafter\def\csname SSlgtEnergyOne256u09\endcsname{$0.1$}
\expandafter\def\csname SSargEnergyOneHalf256u09\endcsname{1.8}
\expandafter\def\csname SSrngEnergyOneHalf256u09\endcsname{$4.9\times 10^{3}$}
\expandafter\def\csname SSlgtEnergyOneHalf256u09\endcsname{$-0.3$}
\expandafter\def\csname SSargSliced256u09\endcsname{1.8}
\expandafter\def\csname SSrngSliced256u09\endcsname{$182.2$}
\expandafter\def\csname SSlgtSliced256u09\endcsname{$-0.3$}
\expandafter\def\csname SSargHill256u09\endcsname{1.2}
\expandafter\def\csname SSrngHill256u09\endcsname{$1.1$}
\expandafter\def\csname SSlgtHill256u09\endcsname{$0.6$}
\expandafter\def\csname SSargEnergyHalf256u099\endcsname{3.5}
\expandafter\def\csname SSrngEnergyHalf256u099\endcsname{$0.9$}
\expandafter\def\csname SSlgtEnergyHalf256u099\endcsname{$-0.0$}
\expandafter\def\csname SSargEnergyOne256u099\endcsname{3.5}
\expandafter\def\csname SSrngEnergyOne256u099\endcsname{$21.3$}
\expandafter\def\csname SSlgtEnergyOne256u099\endcsname{$-0.1$}
\expandafter\def\csname SSargEnergyOneHalf256u099\endcsname{1.8}
\expandafter\def\csname SSrngEnergyOneHalf256u099\endcsname{$973.6$}
\expandafter\def\csname SSlgtEnergyOneHalf256u099\endcsname{$-0.6$}
\expandafter\def\csname SSargSliced256u099\endcsname{1.8}
\expandafter\def\csname SSrngSliced256u099\endcsname{$62.6$}
\expandafter\def\csname SSlgtSliced256u099\endcsname{$-0.4$}
\expandafter\def\csname SSargHill256u099\endcsname{3.5}
\expandafter\def\csname SSrngHill256u099\endcsname{$0.5$}
\expandafter\def\csname SSlgtHill256u099\endcsname{$-0.2$}
\expandafter\def\csname SSargEnergyHalf1024u05\endcsname{1.5}
\expandafter\def\csname SSrngEnergyHalf1024u05\endcsname{$16.1$}
\expandafter\def\csname SSlgtEnergyHalf1024u05\endcsname{$2.2$}
\expandafter\def\csname SSargEnergyOne1024u05\endcsname{1.5}
\expandafter\def\csname SSrngEnergyOne1024u05\endcsname{$764.5$}
\expandafter\def\csname SSlgtEnergyOne1024u05\endcsname{$2.3$}
\expandafter\def\csname SSargEnergyOneHalf1024u05\endcsname{1.5}
\expandafter\def\csname SSrngEnergyOneHalf1024u05\endcsname{$4.1\times 10^{5}$}
\expandafter\def\csname SSlgtEnergyOneHalf1024u05\endcsname{$1.7$}
\expandafter\def\csname SSargSliced1024u05\endcsname{2.8}
\expandafter\def\csname SSrngSliced1024u05\endcsname{$3.7\times 10^{3}$}
\expandafter\def\csname SSlgtSliced1024u05\endcsname{$0.0$}
\expandafter\def\csname SSargHill1024u05\endcsname{1.5}
\expandafter\def\csname SSrngHill1024u05\endcsname{$6.1$}
\expandafter\def\csname SSlgtHill1024u05\endcsname{$3.8$}
\expandafter\def\csname SSargEnergyHalf1024u09\endcsname{3.5}
\expandafter\def\csname SSrngEnergyHalf1024u09\endcsname{$6.5$}
\expandafter\def\csname SSlgtEnergyHalf1024u09\endcsname{$-0.4$}
\expandafter\def\csname SSargEnergyOne1024u09\endcsname{2.8}
\expandafter\def\csname SSrngEnergyOne1024u09\endcsname{$2.9\times 10^{4}$}
\expandafter\def\csname SSlgtEnergyOne1024u09\endcsname{$-0.0$}
\expandafter\def\csname SSargEnergyOneHalf1024u09\endcsname{2.2}
\expandafter\def\csname SSrngEnergyOneHalf1024u09\endcsname{$5.9\times 10^{7}$}
\expandafter\def\csname SSlgtEnergyOneHalf1024u09\endcsname{$0.1$}
\expandafter\def\csname SSargSliced1024u09\endcsname{2.2}
\expandafter\def\csname SSrngSliced1024u09\endcsname{$1.4\times 10^{5}$}
\expandafter\def\csname SSlgtSliced1024u09\endcsname{$-0.1$}
\expandafter\def\csname SSargHill1024u09\endcsname{1.2}
\expandafter\def\csname SSrngHill1024u09\endcsname{$2.9$}
\expandafter\def\csname SSlgtHill1024u09\endcsname{$1.4$}
\expandafter\def\csname SSargEnergyHalf1024u099\endcsname{3.5}
\expandafter\def\csname SSrngEnergyHalf1024u099\endcsname{$1.9$}
\expandafter\def\csname SSlgtEnergyHalf1024u099\endcsname{$-0.6$}
\expandafter\def\csname SSargEnergyOne1024u099\endcsname{2.8}
\expandafter\def\csname SSrngEnergyOne1024u099\endcsname{$3.4\times 10^{3}$}
\expandafter\def\csname SSlgtEnergyOne1024u099\endcsname{$-0.4$}
\expandafter\def\csname SSargEnergyOneHalf1024u099\endcsname{2.2}
\expandafter\def\csname SSrngEnergyOneHalf1024u099\endcsname{$2.4\times 10^{6}$}
\expandafter\def\csname SSlgtEnergyOneHalf1024u099\endcsname{$-0.1$}
\expandafter\def\csname SSargSliced1024u099\endcsname{2.2}
\expandafter\def\csname SSrngSliced1024u099\endcsname{$1.6\times 10^{4}$}
\expandafter\def\csname SSlgtSliced1024u099\endcsname{$-0.1$}
\expandafter\def\csname SSargHill1024u099\endcsname{1.0}
\expandafter\def\csname SSrngHill1024u099\endcsname{$0.9$}
\expandafter\def\csname SSlgtHill1024u099\endcsname{$-0.2$}
\expandafter\def\csname SSargEnergyHalf4096u05\endcsname{1.5}
\expandafter\def\csname SSrngEnergyHalf4096u05\endcsname{$37.1$}
\expandafter\def\csname SSlgtEnergyHalf4096u05\endcsname{$4.1$}
\expandafter\def\csname SSargEnergyOne4096u05\endcsname{1.5}
\expandafter\def\csname SSrngEnergyOne4096u05\endcsname{$328.2$}
\expandafter\def\csname SSlgtEnergyOne4096u05\endcsname{$4.3$}
\expandafter\def\csname SSargEnergyOneHalf4096u05\endcsname{1.5}
\expandafter\def\csname SSrngEnergyOneHalf4096u05\endcsname{$5.9\times 10^{4}$}
\expandafter\def\csname SSlgtEnergyOneHalf4096u05\endcsname{$2.0$}
\expandafter\def\csname SSargSliced4096u05\endcsname{1.8}
\expandafter\def\csname SSrngSliced4096u05\endcsname{$1.3\times 10^{3}$}
\expandafter\def\csname SSlgtSliced4096u05\endcsname{$0.2$}
\expandafter\def\csname SSargHill4096u05\endcsname{1.5}
\expandafter\def\csname SSrngHill4096u05\endcsname{$12.6$}
\expandafter\def\csname SSlgtHill4096u05\endcsname{$6.8$}
\expandafter\def\csname SSargEnergyHalf4096u09\endcsname{1.8}
\expandafter\def\csname SSrngEnergyHalf4096u09\endcsname{$11.3$}
\expandafter\def\csname SSlgtEnergyHalf4096u09\endcsname{$0.7$}
\expandafter\def\csname SSargEnergyOne4096u09\endcsname{1.8}
\expandafter\def\csname SSrngEnergyOne4096u09\endcsname{$365.3$}
\expandafter\def\csname SSlgtEnergyOne4096u09\endcsname{$1.3$}
\expandafter\def\csname SSargEnergyOneHalf4096u09\endcsname{1.5}
\expandafter\def\csname SSrngEnergyOneHalf4096u09\endcsname{$1.3\times 10^{5}$}
\expandafter\def\csname SSlgtEnergyOneHalf4096u09\endcsname{$0.6$}
\expandafter\def\csname SSargSliced4096u09\endcsname{3.5}
\expandafter\def\csname SSrngSliced4096u09\endcsname{$2.0\times 10^{3}$}
\expandafter\def\csname SSlgtSliced4096u09\endcsname{$-0.1$}
\expandafter\def\csname SSargHill4096u09\endcsname{1.5}
\expandafter\def\csname SSrngHill4096u09\endcsname{$6.1$}
\expandafter\def\csname SSlgtHill4096u09\endcsname{$3.1$}
\expandafter\def\csname SSargEnergyHalf4096u099\endcsname{1.2}
\expandafter\def\csname SSrngEnergyHalf4096u099\endcsname{$0.8$}
\expandafter\def\csname SSlgtEnergyHalf4096u099\endcsname{$-0.1$}
\expandafter\def\csname SSargEnergyOne4096u099\endcsname{2.8}
\expandafter\def\csname SSrngEnergyOne4096u099\endcsname{$48.4$}
\expandafter\def\csname SSlgtEnergyOne4096u099\endcsname{$-0.1$}
\expandafter\def\csname SSargEnergyOneHalf4096u099\endcsname{3.5}
\expandafter\def\csname SSrngEnergyOneHalf4096u099\endcsname{$4.9\times 10^{3}$}
\expandafter\def\csname SSlgtEnergyOneHalf4096u099\endcsname{$-0.1$}
\expandafter\def\csname SSargSliced4096u099\endcsname{3.5}
\expandafter\def\csname SSrngSliced4096u099\endcsname{$202.0$}
\expandafter\def\csname SSlgtSliced4096u099\endcsname{$-0.2$}
\expandafter\def\csname SSargHill4096u099\endcsname{1.2}
\expandafter\def\csname SSrngHill4096u099\endcsname{$2.2$}
\expandafter\def\csname SSlgtHill4096u099\endcsname{$0.6$}
\expandafter\def\csname SSargEnergyHalf16384u05\endcsname{1.5}
\expandafter\def\csname SSrngEnergyHalf16384u05\endcsname{$86.0$}
\expandafter\def\csname SSlgtEnergyHalf16384u05\endcsname{$10.8$}
\expandafter\def\csname SSargEnergyOne16384u05\endcsname{1.5}
\expandafter\def\csname SSrngEnergyOne16384u05\endcsname{$672.3$}
\expandafter\def\csname SSlgtEnergyOne16384u05\endcsname{$9.8$}
\expandafter\def\csname SSargEnergyOneHalf16384u05\endcsname{1.5}
\expandafter\def\csname SSrngEnergyOneHalf16384u05\endcsname{$9.9\times 10^{5}$}
\expandafter\def\csname SSlgtEnergyOneHalf16384u05\endcsname{$6.2$}
\expandafter\def\csname SSargSliced16384u05\endcsname{1.5}
\expandafter\def\csname SSrngSliced16384u05\endcsname{$6.2\times 10^{4}$}
\expandafter\def\csname SSlgtSliced16384u05\endcsname{$1.0$}
\expandafter\def\csname SSargHill16384u05\endcsname{1.5}
\expandafter\def\csname SSrngHill16384u05\endcsname{$29.5$}
\expandafter\def\csname SSlgtHill16384u05\endcsname{$17.7$}
\expandafter\def\csname SSargEnergyHalf16384u09\endcsname{1.5}
\expandafter\def\csname SSrngEnergyHalf16384u09\endcsname{$28.7$}
\expandafter\def\csname SSlgtEnergyHalf16384u09\endcsname{$3.5$}
\expandafter\def\csname SSargEnergyOne16384u09\endcsname{1.5}
\expandafter\def\csname SSrngEnergyOne16384u09\endcsname{$2.5\times 10^{3}$}
\expandafter\def\csname SSlgtEnergyOne16384u09\endcsname{$4.8$}
\expandafter\def\csname SSargEnergyOneHalf16384u09\endcsname{1.5}
\expandafter\def\csname SSrngEnergyOneHalf16384u09\endcsname{$3.7\times 10^{6}$}
\expandafter\def\csname SSlgtEnergyOneHalf16384u09\endcsname{$1.5$}
\expandafter\def\csname SSargSliced16384u09\endcsname{2.2}
\expandafter\def\csname SSrngSliced16384u09\endcsname{$2.7\times 10^{4}$}
\expandafter\def\csname SSlgtSliced16384u09\endcsname{$-0.5$}
\expandafter\def\csname SSargHill16384u09\endcsname{1.5}
\expandafter\def\csname SSrngHill16384u09\endcsname{$12.0$}
\expandafter\def\csname SSlgtHill16384u09\endcsname{$6.9$}
\expandafter\def\csname SSargEnergyHalf16384u099\endcsname{2.8}
\expandafter\def\csname SSrngEnergyHalf16384u099\endcsname{$2.0$}
\expandafter\def\csname SSlgtEnergyHalf16384u099\endcsname{$0.7$}
\expandafter\def\csname SSargEnergyOne16384u099\endcsname{2.8}
\expandafter\def\csname SSrngEnergyOne16384u099\endcsname{$508.8$}
\expandafter\def\csname SSlgtEnergyOne16384u099\endcsname{$0.7$}
\expandafter\def\csname SSargEnergyOneHalf16384u099\endcsname{2.2}
\expandafter\def\csname SSrngEnergyOneHalf16384u099\endcsname{$1.8\times 10^{5}$}
\expandafter\def\csname SSlgtEnergyOneHalf16384u099\endcsname{$-0.5$}
\expandafter\def\csname SSargSliced16384u099\endcsname{2.2}
\expandafter\def\csname SSrngSliced16384u099\endcsname{$2.8\times 10^{3}$}
\expandafter\def\csname SSlgtSliced16384u099\endcsname{$-0.5$}
\expandafter\def\csname SSargHill16384u099\endcsname{1.5}
\expandafter\def\csname SSrngHill16384u099\endcsname{$3.4$}
\expandafter\def\csname SSlgtHill16384u099\endcsname{$1.6$}
\expandafter\def\csname SSargEnergyHalf65536u05\endcsname{1.5}
\expandafter\def\csname SSrngEnergyHalf65536u05\endcsname{$186.4$}
\expandafter\def\csname SSlgtEnergyHalf65536u05\endcsname{$22.8$}
\expandafter\def\csname SSargEnergyOne65536u05\endcsname{1.5}
\expandafter\def\csname SSrngEnergyOne65536u05\endcsname{$2.0\times 10^{3}$}
\expandafter\def\csname SSlgtEnergyOne65536u05\endcsname{$25.0$}
\expandafter\def\csname SSargEnergyOneHalf65536u05\endcsname{1.5}
\expandafter\def\csname SSrngEnergyOneHalf65536u05\endcsname{$9.7\times 10^{6}$}
\expandafter\def\csname SSlgtEnergyOneHalf65536u05\endcsname{$15.8$}
\expandafter\def\csname SSargSliced65536u05\endcsname{1.8}
\expandafter\def\csname SSrngSliced65536u05\endcsname{$1.2\times 10^{5}$}
\expandafter\def\csname SSlgtSliced65536u05\endcsname{$-0.1$}
\expandafter\def\csname SSargHill65536u05\endcsname{1.5}
\expandafter\def\csname SSrngHill65536u05\endcsname{$67.0$}
\expandafter\def\csname SSlgtHill65536u05\endcsname{$40.6$}
\expandafter\def\csname SSargEnergyHalf65536u09\endcsname{1.5}
\expandafter\def\csname SSrngEnergyHalf65536u09\endcsname{$71.4$}
\expandafter\def\csname SSlgtEnergyHalf65536u09\endcsname{$8.8$}
\expandafter\def\csname SSargEnergyOne65536u09\endcsname{1.5}
\expandafter\def\csname SSrngEnergyOne65536u09\endcsname{$405.9$}
\expandafter\def\csname SSlgtEnergyOne65536u09\endcsname{$12.0$}
\expandafter\def\csname SSargEnergyOneHalf65536u09\endcsname{1.5}
\expandafter\def\csname SSrngEnergyOneHalf65536u09\endcsname{$9.2\times 10^{3}$}
\expandafter\def\csname SSlgtEnergyOneHalf65536u09\endcsname{$6.3$}
\expandafter\def\csname SSargSliced65536u09\endcsname{1.8}
\expandafter\def\csname SSrngSliced65536u09\endcsname{$812.7$}
\expandafter\def\csname SSlgtSliced65536u09\endcsname{$-0.6$}
\expandafter\def\csname SSargHill65536u09\endcsname{1.5}
\expandafter\def\csname SSrngHill65536u09\endcsname{$25.6$}
\expandafter\def\csname SSlgtHill65536u09\endcsname{$14.9$}
\expandafter\def\csname SSargEnergyHalf65536u099\endcsname{1.5}
\expandafter\def\csname SSrngEnergyHalf65536u099\endcsname{$5.6$}
\expandafter\def\csname SSlgtEnergyHalf65536u099\endcsname{$2.4$}
\expandafter\def\csname SSargEnergyOne65536u099\endcsname{1.5}
\expandafter\def\csname SSrngEnergyOne65536u099\endcsname{$65.5$}
\expandafter\def\csname SSlgtEnergyOne65536u099\endcsname{$3.6$}
\expandafter\def\csname SSargEnergyOneHalf65536u099\endcsname{1.8}
\expandafter\def\csname SSrngEnergyOneHalf65536u099\endcsname{$723.6$}
\expandafter\def\csname SSlgtEnergyOneHalf65536u099\endcsname{$1.6$}
\expandafter\def\csname SSargSliced65536u099\endcsname{1.8}
\expandafter\def\csname SSrngSliced65536u099\endcsname{$80.6$}
\expandafter\def\csname SSlgtSliced65536u099\endcsname{$-0.6$}
\expandafter\def\csname SSargHill65536u099\endcsname{1.5}
\expandafter\def\csname SSrngHill65536u099\endcsname{$7.7$}
\expandafter\def\csname SSlgtHill65536u099\endcsname{$5.1$}

\newcommand{\SSRefAlpha}{1.5}

\newcommand{\TransportAlphaOne}{2.50}
\newcommand{\TransportAlphaTwo}{1.76}
\newcommand{\TransportSamples}{16384}
\newcommand{\TransportArrows}{140}
\newcommand{\TransportWindowMass}{94}

\title{Attention Kernels for Learning
Maps Between Heavy-Tailed Measures}
\workshoptitle{AI for Stochastic Dynamics: From Theoretical Foundations to Scientific Applications}

\author{%
Kailen Hargenrader \\
Computing and Mathematical Sciences \\
California Institute of Technology \\
\texttt{khargenr@caltech.edu} \\
\And
Edoardo Calvello \\
Lawrence Berkeley National Laboratory \\
University of California, Berkeley \\
ICSI\\
\texttt{ecalvello@lbl.gov} \\
\And
Bohan Chen \\
Computing and Mathematical Sciences \\
California Institute of Technology \\
\texttt{bhchen@caltech.edu} \\
}

\begin{document}

\maketitle

\begin{abstract}
Operator learning on probability measures can be accomplished with transformers. For measures with polynomial tails, the exponential weighting in softmax can make the corresponding measure-level attention integrals diverge. This motivates replacing the exponential with slower-growing functions. We construct two benchmarks for operator learning on measures with closed-form targets. We use these benchmarks to study attention kernel growth and data transformation in post-norm transformers. Without data transformation, the softmax models exhibit ensemble collapse on both heavy-tailed benchmarks, while the three slower-growing kernels avoid collapse. Symlog preprocessing allows softmax to avoid collapse on the matrix inverse task but not on the sheared swap task. On the Gaussian control, all four kernels perform similarly. We also examine how sample size affects the sensitivity of empirical energy and Wasserstein distances to tail differences. These results support slower-growing attention kernels as an effective design choice for post-norm transformers learning from heavy-tailed ensembles.
\end{abstract}

\section{Introduction}
\label{sec:intro}

In Bayesian inference, learning a map from prior ensembles to posterior ensembles is an example of operator learning on probability measures. Drawing a posterior ensemble often requires multiple calls to the forward model, a cost that can be amortized by a learned surrogate. Transformers are suitable surrogates for these maps because they are composed of softmax attention layers, which are naturally interpretable as operators on measures \citep{vuckovic2020mathematical}. Permutation-equivariant transformers, such as non-causal transformers with no positional encoding, correspond to maps on unordered ensembles sampled from measures, enabling the action of the transformer on ensembles to be formulated via empirical measures.

Applications of Bayesian inference, including
weather forecasting, finance, and turbulent flow, can involve heavy-tailed probability measures. In these domains, extreme values can represent important events, such as severe weather, defaults, or velocity spikes. These applications motivate studying transformers as operators on \emph{heavy-tailed} probability measures.

Written as an operator on measures, softmax attention is a ratio of integrals with exponential weights. Its integrability depends on the exponential moments of the input measure along the score directions \citep{vuckovic2020mathematical,bach2026enhanced}. For the mirrored Pareto family studied here, a nonzero linear score makes the normalizing integral diverge. Replacing the exponential by a slower-growing positive function weakens the required tail integrability and motivates our empirical study of kernel choice.

\subsection{Related work}
\label{sec:relatedwork}

We summarize two lines of research related to this work.

\paragraph{Attention as an operator on measures.}
Attention was introduced as a sequence-to-sequence mapping \citep{vaswani2017attention}, but can be interpreted as an expectation under a probability distribution over sequence indices \citep{calvello2025continuum}. This interpretation is useful to understand earlier studies that view attention as an operator on probability measures \citep{vuckovic2020mathematical}, a view that connected stacked attention dynamics to interacting-particle systems \citep{geshkovski2023clusters, geshkovski2025mathematical}. When applied to samples from an underlying reference measure, architectures must be permutation invariant or equivariant to be operators on measures. This is satisfied by construction in deep sets
and set transformer \citep{zaheer2017deep, lee2019set}. Such architectures have been used as ensemble
update and postprocessing maps in data assimilation
\citep{hohlein2024postprocessing, bach2026enhanced, bach2026probabilistic} and
can parametrize maps on the space of measures \citep{huang2025mfg, cole2026incontext}, unlike
neural operators, which learn maps between function spaces
\citep{lu2021deeponet, li2021fourier}. Other works have explored learning measure-to-measure maps to amortize computation across problem instances; see for example \citep{vandergrift2026measure, tsimpos2026conditioning}. 

\paragraph{Attention kernels.}
Other works have explored replacing the softmax kernel in attention, primarily for efficiency. Replacement kernels include $\mathrm{Sigmoid}$ \citep{ramapuram2025sigmoid}, $\mathrm{ReLU}$ \citep{wortsman2023replacing, shen2023study},
$\mathrm{ReLU}^2$ \citep{hua2022transformer}, and $x^3$ \citep{saratchandran2024rethinking}.

\subsection{Contributions}
\label{sec:contribution}

Our contributions are organized into three categories.
\begin{itemize}
    \item \textbf{Benchmarks.} We construct two operators on identifiable families of input measures. Both operators provide closed-form targets, enabling the computation of an i.i.d. sampling reference for benchmarking.
    \item \textbf{Ablations.} We study operator learning on heavy-tailed measures using post-norm transformers, ablating attention-kernel growth and data transformation.
    \item \textbf{Tail sensitivity test.} We design a test to quantify how set divergences respond to differences in distribution tails. We use this test to inform our choice of the energy distance sample estimator.
\end{itemize}

\section{Background: Attention on measures and tail integrability}
\label{sec:measureattention}

The attention map as defined in \cite{vaswani2017attention}
acts on sequences $\{u_i\}_{i=1}^{N}\subset\mathbb{R}^{d_u}$ and
$\{v_i\}_{i=1}^{L}\subset\mathbb{R}^{d_v}$, collected into matrices
$u\in\mathbb{R}^{N\times d_u}$, $v\in\mathbb{R}^{L\times d_v}$, with learnable
weights $Q\in\mathbb{R}^{d_K\times d_u}$, $K\in\mathbb{R}^{d_K\times d_v}$,
$V\in\mathbb{R}^{d_V\times d_v}$. Letting $q=uQ^{\top}$, $k=vK^{\top}$, and $z=vV^{\top}$, the scaled score matrix and attention outputs are defined as
\begin{equation}
    S = \frac{qk^{\top}}{\sqrt{d_K}},\quad
    \mathrm{attention}(u,v) = \mathrm{softmax}(S)z,\quad
    \bigl(\mathrm{softmax}(S)\bigr)_{ij}
      = \frac{\exp(S_{ij})}{\sum_{\ell=1}^{L}\exp(S_{i\ell})}.
  \label{eq:attention-discrete}
\end{equation}
Self-attention is the case $v=u$. Each output row is a normalized weighted average of the values, so a sample set enters through its empirical measure. This normalized-expectation view allows attention to be formulated directly as an operator on probability measures \citep{vuckovic2020mathematical}.

Here, $\mathcal{P}(E)$
denotes the Borel probability measures on $E$, and norms of vectors are
Euclidean. Furthermore, for a measurable map $T$, the pushforward $T_\sharp\mu$ is defined as the
law of $T(X)$ when $X$ has law $\mu$.
For a query measure $\mu\in\mathcal{P}(\mathbb{R}^{d_u})$ and a key/value measure $\nu\in\mathcal{P}(\mathbb{R}^{d_v})$, we define the attention operator on measures as
\begin{equation}
  \begin{aligned}
    \mathsf{A}(\mu,\nu) &= \mathfrak{A}(\cdot;\nu)_{\sharp}\mu,\\
    \mathfrak{A}(s;\nu) &= \frac{\int_{\mathbb{R}^{d_v}}\kappa_{Q,K}(s,y)Vy\nu(\mathrm{d}y)}{\int_{\mathbb{R}^{d_v}}\kappa_{Q,K}(s,y)\nu(\mathrm{d}y)}.
  \end{aligned}
  \label{eq:measure-attention}
\end{equation}
Here, $\mathfrak{A}$ maps one query location to its attention output and $\mathsf{A}$ is the induced pushforward on query measures; the exponential kernel $\kappa_{Q,K}(s,y)=\exp(\langle Qs,Ky\rangle/\sqrt{d_K})$ recovers softmax. As described in \cite{bach2026enhanced}, the definition in \eqref{eq:measure-attention} leads to \eqref{eq:attention-discrete} when $\mu$ and $\nu$ are empirical measures defined on the collections of atoms $\{u_i\}_{i=1}^N$ and $\{v_i\}_{i=1}^L$, respectively.

For \eqref{eq:measure-attention} to be well defined, its numerator must be absolutely integrable and its denominator positive and finite:
\begin{equation}
  \begin{aligned}
    \mathcal{N}_{\kappa}(s;\nu) &= \int_{\mathbb{R}^{d_v}}\kappa_{Q,K}(s,y)\lVert Vy\rVert_2\nu(\mathrm{d}y)<\infty,\\
    \mathcal{D}_{\kappa}(s;\nu) &= \int_{\mathbb{R}^{d_v}}\kappa_{Q,K}(s,y)\nu(\mathrm{d}y)\in(0,\infty).
  \end{aligned}
  \label{eq:attention-integrability}
\end{equation}

For softmax, finite exponential moments of $\lVert Y\rVert_2$ for $Y\sim\nu$ at every positive rate are sufficient for \eqref{eq:attention-integrability} to hold for every query and choice of weights. The mirrored Pareto family instead has an infinite normalizing integral whenever the linear score $\langle Qs,Ky\rangle$ is nonzero as a function of $y$. Appendix~\ref{app:integrability} gives the tail conditions and the argument for this family. Finite-token attention sums remain finite for finite scores, so population integrability and finite-ensemble training are distinct questions.

Replacing the exponential by a slower-growing positive function weakens the tail integrability required of $\nu$. We investigate how this choice affects learning from finite ensembles.

\section{Methods}
\label{sec:methods}
We specify the learning problem, model architecture, and training procedure.

\subsection{Setting: operator learning on measures in 2D}
\label{sec:setting}
Let $\Pi$ be a distribution over parameter instances $\omega$. Each instance $\omega\sim \Pi$ determines an input--output measure pair $(\mu_{I}(\omega),\mu_{O}(\omega))\in \mathcal{P}(\mathbb{R}^{2})\times \mathcal{P}(\mathbb{R}^{2})$, where the output measure is related to the input measure $\mu_{O}(\omega) = \Psi(\mu_{I}(\omega))$ through an operator $\Psi:\mathcal{M}_I\to\mathcal{P}(\mathbb{R}^{2})$, defined on the input family $\mathcal{M}_I=\{\mu_I(\omega):\omega\in\operatorname{supp}\Pi\}$. We use the identifiable operator--family combinations established in Appendix~\ref{app:gaussid}. When the distinction between input and output is irrelevant, a parameterized measure is denoted $\mu(\omega)$ without a subscript. We approximate $\Psi$ over a family of measures denoted by $\mathcal{M}(\mu,\Pi)$.

In practical settings, $\mu(\omega)\in \mathcal{M}(\mu, \Pi)$ is not available in closed form, but as a collection of samples. We represent each measure by a fixed number of samples $N$ and approximate $\Psi$ empirically with a learned model $\widetilde{\Psi}_\theta:\mathbb{R}^{N\times2}\to\mathbb{R}^{N\times2}$. We train $\widetilde{\Psi}_\theta$ by minimizing the following training objective with respect to $\theta$: 
\begin{equation}
  \mathcal{L}(\theta)
    = \mathbb{E}_{\omega\sim\Pi}\;
      \mathbb{E}_{X\sim{\mu_{I}(\omega)}^{\otimes N},\,
                  Y\sim{\mu_{O}(\omega)}^{\otimes N}}
      \bigl[D\bigl(\widetilde{\Psi}_{\theta}(X),\,Y\bigr)\bigr],
  \label{eq:population-risk}
\end{equation}
where $D$ denotes a set divergence, and $X = [x_1, \dots, x_N]^\top \in \mathbb{R}^{N \times 2}$ and $Y = [y_1, \dots, y_N]^\top \in \mathbb{R}^{N \times 2}$ consist of $N$ i.i.d. samples drawn in closed form from $\mu_I(\omega)$ and $\mu_O(\omega)$, respectively (denoted succinctly as $X \sim \mu_I(\omega)^{\otimes N}$ and $Y \sim \mu_O(\omega)^{\otimes N}$). The outer expectation in \eqref{eq:population-risk} averages over parameter instances, so the model learns the operator across a family of measures. Appendix~\ref{app:blocking} discusses the role of the fixed ensemble size.

\subsection{Learning problem}
\label{sec:learningproblem}
For the benchmarks in Section~\ref{sec:ablation}, the measures in $\mathcal{M}(\mu,\Pi)$ take the form $\mu(\omega) = T({\omega_T})_{\sharp}\eta({\omega_\eta})$, where $\omega = \{{\omega_T}, {\omega_\eta}\}$, $T$ is a linear transformation, and $\eta=\eta^{(1)} \otimes\eta^{(2)}$ is a product of independent 1D marginal distributions. $\Pi$ decomposes into $\Pi_T \otimes \Pi_\eta$ such that $\omega_T \sim \Pi_T$ and $\omega_\eta \sim \Pi_\eta$. In Section~\ref{sec:heavylightmarginals}, we define one heavy-tailed and one light-tailed family of distributions used for $\eta$ and specify how $\Pi_\eta$ is chosen.  In Section~\ref{sec:operators}, we define the two operators used for $\Psi$. We defer our choices of $\Pi_T$ to Appendix~\ref{app:shearparams}.

\subsubsection{Heavy-tailed and light-tailed marginals \texorpdfstring{$\eta$}{eta}}
\label{sec:heavylightmarginals}
We use mirrored Pareto marginals with polynomial tails and Gaussian marginals as a light-tailed control.

The heavy-tailed family is a \emph{mirrored Pareto} distribution.
A $\mathrm{Pareto}(\lambda,a)$ variable, with positive scale $\lambda$
and positive tail index $a$, is supported on $[\lambda,\infty)$ with survival function and density
\begin{equation}
  S_P(x)=(x/\lambda)^{-a},
  \qquad
  f_P(x)=a\lambda^{a}x^{-a-1},
  \qquad x\ge\lambda.
  \label{eq:Pareto}
\end{equation}

The Pareto density is discontinuous at the edge of its support.
Mirroring about $\lambda$ gives a continuous density supported on
$\mathbb{R}$. Let $Z\sim\mathrm{Pareto}(\lambda,a)$ and let
$\mathrm{Rad}$ be an independent Rademacher sign, equally likely to
be positive or negative. The mirrored variable $X$ and its density $f$
are
\begin{equation}
  X=\begin{cases} Z, & \mathrm{Rad}=+1,\\ 2\lambda-Z, & \mathrm{Rad}=-1,\end{cases}
  \qquad
  f(z)=\tfrac{1}{2}f_P\bigl(m(z)\bigr),
  \quad m(z)=\lambda+|z-\lambda|,
  \label{eq:mirrored}
\end{equation}
where $m$ folds either branch onto the Pareto support.
The mirrored Pareto distribution is heavy-tailed, with polynomial decay.

For light tails, we use a Gaussian distribution $\mathcal{N}(\mu_{\mathcal{N}}, \sigma^2)$. Its tail satisfies the sufficient bound in \eqref{eq:tailclass}. Appendix~\ref{app:mirrored} visualizes these distributions.

In Section~\ref{sec:ablation}, we fix $\lambda = 1$ and $\mu_{\mathcal{N}} = 1$. We allow $\omega_\eta \sim \Pi_\eta = \mathcal{U}([2,3] \times [1 + \epsilon ,2])$, with $\epsilon > 0$, to control the spread of each marginal of $\eta$, corresponding to $a$ for the mirrored Pareto distribution and $\sigma$ for the Gaussian. For the Gaussian the two ranges are exchanged, allowing the second axis to always represent the marginal with greater spread.
For the mirrored Pareto distribution, the interval $[1 + \epsilon ,2]$ gives a finite first
moment and an infinite variance. Our training
objective \eqref{eq:population-risk} uses the V-statistic energy distance \eqref{eq:energy}.
The corresponding population formula \eqref{eq:energy-pop} requires finite first
moments of both measures at the training exponent. These hold for the target
measures at every fixed parameter instance. During our ablation in Section \ref{sec:ablation}, $1 + \epsilon$ is the lowest value above $1$ afforded by 32-bit floating point precision. Appendix~\ref{app:momentcondition} proves our claim about the existence of moments for the Pareto marginal.

\subsubsection{Operators and parameter distributions}
\label{sec:operators}
We use the following two operators as benchmarks for the ablation in Section~\ref{sec:ablation}.%
\paragraph{Sheared swap operator.}
Let $M({\omega_T})$ be the product of two shear matrices as described in \eqref{eq:shearparams} and $P$ the 2D exchange matrix $\left(\begin{smallmatrix} 0 & 1 \\ 1 & 0 \end{smallmatrix}\right)$. Let the input measure $\mu_I(\omega) = M({\omega_T})_\sharp\eta({\omega_\eta})$ and the output measure $\mu_O(\omega) = \bigl(M({\omega_T})P\bigr)_\sharp\eta({\omega_\eta})$. This defines the mapping
\[
\Psi:\mu_I(\omega) \mapsto \bigl(M({\omega_T})PM^{-1}({\omega_T})\bigr)_\sharp \mu_I(\omega).
\]
\paragraph{Matrix inverse operator.}
Let $B({\omega_T})\in GL^{+}(2)$, where $B$ is parameterized using the decomposition in \eqref{eq:iwasawa}. Let the input measure $\mu_I(\omega) = B({\omega_T})_\sharp\eta({\omega_\eta})$ and the output measure $\mu_O(\omega) = B({\omega_T})^{-1}_\sharp\eta({\omega_\eta})$. This defines the mapping
\[
\Psi:\mu_I(\omega) \mapsto B({\omega_T})^{-2}_\sharp \mu_I(\omega).
\]
Appendix~\ref{app:shearparams} specifies $M({\omega_T})$, $B({\omega_T})$, and $\Pi_T$. Appendix~\ref{app:opchoice} explains the choice of operators and discusses identifiability.

\subsection{Model architecture}
We approximate $\Psi$ with a six-layer non-causal transformer. Sections~\ref{sec:family} and~\ref{sec:norm} define its attention kernels and post-norm blocks. Appendix~\ref{app:model} gives the full architecture and compute requirements.

\subsubsection{Attention kernel suite}
\label{sec:family}
Recall the attention formulation of \citep{vaswani2017attention} from Section~\ref{sec:measureattention}. 
We replace the exponential in softmax by a scalar map $g$ applied elementwise and compute
\begin{equation}
  \begin{aligned}
    A_g(S)_{ij} &= \frac{g(S_{ij})}{\sum_{\ell=1}^{L}g(S_{i\ell})},\\
    \mathrm{attention}_g(u,v) &= A_g(S)z.
  \end{aligned}
  \label{eq:kernel-attention}
\end{equation}
Softmax is the special case $g(s)=e^s$. Our attention kernel suite only admits strictly positive maps $g$, ensuring that a finite normalizer cannot vanish. In the positive domain, our kernels span exponential to logarithmic growth. Table~\ref{tab:family} lists the suite, containing softmax as well as three ELU-style maps inspired by the $1+\mathrm{elu}$ map of \citet{katharopoulos2020transformers}. Appendix~\ref{app:kernelviz} visualizes the suite.

\begin{table}[h]
  \centering
  \caption{The attention kernel suite, ordered by the growth rate of the positive branch,
  given in the last column as $s\to\infty$. Every member is strictly positive and
  $C^1$ at the origin with $g(0)=1$.}
  \label{tab:family}
  \small
  \begin{tabular}{llll}
    \toprule
    kernel & $g(s)$, $s>0$ & $g(s)$, $s\leq0$ & growth \\
    \midrule
    softmax      & $e^{s}$              & $e^{s}$      & exponential \\
    squared-ELU  & $(s+1)^2$            & $e^{2s}$     & quadratic \\
    linear-ELU   & $s+1$                & $e^{s}$      & linear \\
    log-ELU      & $\log(1+s)+1$        & $(1-s)^{-1}$ & logarithmic \\
    \bottomrule
  \end{tabular}
\end{table}

\subsubsection{Post-norm blocks}
\label{sec:norm}
The main experiments use post-norm blocks, which apply layer normalization after each residual addition. With residual scaling inspired by \citet{wang2022deepnet}, an attention sublayer computes
\[
  x\mapsto\mathrm{LN}\bigl(\alpha_{\mathrm{RS}}x+\mathrm{Attn}(x)\bigr),
  \qquad \alpha_{\mathrm{RS}}=4.899.
\]
The feed-forward sublayer uses the same residual structure.

\subsection{Training}
Section~\ref{sec:logtransform} describes the symlog preprocessing ablated in Section~\ref{sec:ablation}. Section~\ref{sec:objective} defines the squared energy distance and its V-statistic estimator used for $D$.

\subsubsection{Symlog data transformation}
\label{sec:logtransform}
Compressing heavy-tailed data by a logarithm is standard practice
\citep{resnick2007heavy,nair2022fundamentals}. We extend the logarithm to the negative domain while preserving sign, implemented as the \emph{symlog} transform of \citet{webber2013bisymmetric,hafner2023dreamerv3},
\begin{equation}
  \sigma(x)=\operatorname{sign}(x)\log(1+|x|),
  \qquad
  \sigma^{-1}(y)=\operatorname{sign}(y)\bigl(e^{|y|}-1\bigr).
  \label{eq:signed-log-main}
\end{equation}
We apply $\sigma$ to the input samples and $\sigma^{-1}$ to the transformer output. Writing $f_\theta$ for the transformer, the complete predictor is $\widetilde{\Psi}_\theta=\sigma^{-1}\circ f_\theta\circ\sigma$, with $\sigma=\mathrm{id}$ when preprocessing is not used.

\subsubsection{Energy distance}
\label{sec:objective}
We use the \emph{squared} energy distance with $p = 1$ \citep{szekely2004testing,szekely2013energy} as the divergence $D$ in \eqref{eq:population-risk}. For measures $\mu,\nu$ on
$\mathbb{R}^{d}$ with finite $p$th absolute moments for both measures and $p\in(0,2)$, the squared
energy distance is
\begin{equation}
  \mathcal{E}_p^{2}(\mu,\nu)
    = 2\,\mathbb{E}\|X-Y\|^{p}
      - \mathbb{E}\|X-X'\|^{p}
      - \mathbb{E}\|Y-Y'\|^{p}
    \;\ge\; 0,
  \label{eq:energy-pop}
\end{equation}
where $X,X'$ are independent draws from $\mu$ and $Y,Y'$ independent draws from
$\nu$. For sample
sets $U,V$ of size $N$ we use the V-statistic estimator of
\eqref{eq:energy-pop},
\begin{equation}
  \mathcal{E}_p^{2}(U,V)=\frac{1}{N^{2}}\Bigl(
      2\sum_{i,j}\|u_i-v_j\|^{p}
      -\sum_{i,j}\|u_i-u_j\|^{p}
      -\sum_{i,j}\|v_i-v_j\|^{p}\Bigr),
  \label{eq:energy}
\end{equation}
and set $D=\mathcal{E}_1^{2}$. For conciseness, we refer to this squared distance as the \emph{energy distance} below.

\paragraph{The i.i.d. sampling reference.}
Because \eqref{eq:energy} is a finite-sample V-statistic, two independent size-$N$ ensembles from the same nondegenerate target have a positive expected distance. We use their mean or median distance as an \emph{i.i.d. sampling reference}, denoted $\mathrm{floor}_N$ in the formulas.
\begin{equation}
  \mathrm{floor}_{N}
    = \mathrm{summary}\;
      \mathcal{E}_1^{2}\bigl(Y_{\omega},\,Y'_{\omega}\bigr),
  \qquad
  Y_{\omega},Y'_{\omega}\overset{\mathrm{iid}}{\sim}{\mu_{O}(\omega)}^{\otimes N},
  \label{eq:noise-floor}
\end{equation}
where $\mathrm{summary}$ is the mean or median over trials, each drawing $\omega\sim\Pi$ and two independent target ensembles. The reference is computed separately for each operator and marginal family.

\section{Experiments}
\label{sec:experiments}

Section~\ref{sec:tailsensitivity} examines how the sample budget $N$ affects the sensitivity of the energy distance to tail differences. This experiment informs our choice of $N$ when performing the ablations in Section~\ref{sec:ablation}.

\subsection{Tail sensitivity of the energy distance}
\label{sec:tailsensitivity}
We first examine how the empirical energy distance responds to differences in polynomial tail decay. The divergence $D$ should distinguish when measures decay at different rates in their low density tail regions, even when their high density regions are equivalent. The population energy distance is zero exactly when the two measures agree.

\begin{figure}[!htbp]
  \centering
  \includegraphics[width=\linewidth]{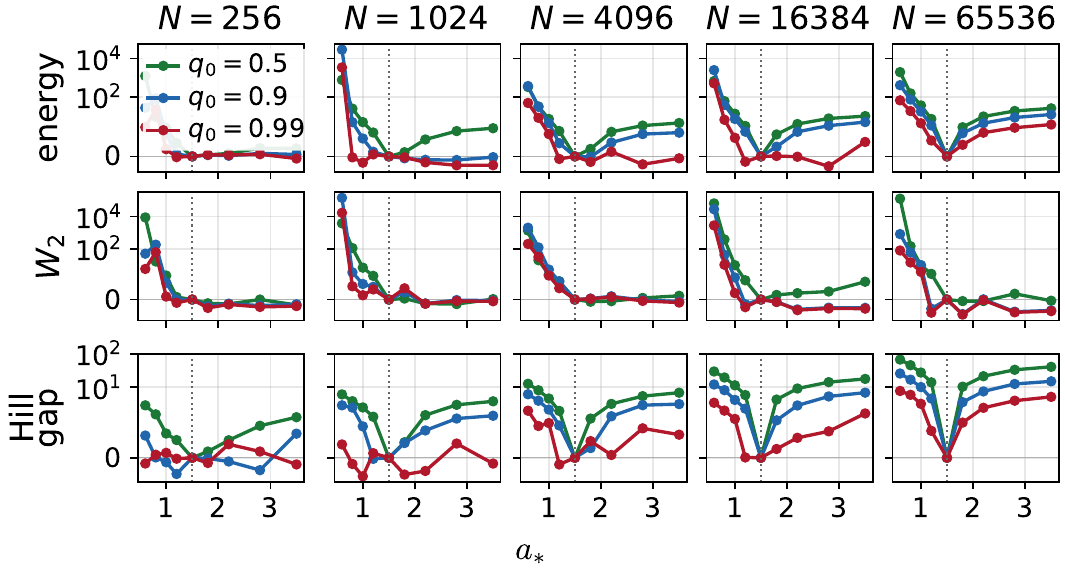}
  \caption{Sensitivity of the energy distance, 2-Wasserstein distance, and Hill gap to tail-index differences at varying sample budgets $N$. Each plotted point reports a value computed from two sample sets drawn from a measure pair as described in Section~\ref{sec:tailsensitivity}. The rows report the energy distance, 2-Wasserstein distance, and absolute difference between the two sample Hill estimates (Hill gap), respectively. Each plot displays the value against $a_*$. Every value is averaged over multiple draws, then shifted and scaled as described in Appendix~\ref{app:noiseunits}, so that $a_* = a_{\text{ref}}$ corresponds to a value of 0. The rows use separate vertical scales. The colors denote the quantile value $q_0$ of the seam.}
  \label{fig:metricres}
\end{figure}

To compare empirical divergences, we evaluate them on sample sets drawn from a testbed of 1D measure pairs. Each measure in the testbed follows a piecewise density function that joins two Pareto densities at a specific quantile, which we call the seam. A Pareto density decreases monotonically on its support, causing the Pareto components below and above the seam to control the high density and low density regions of the piecewise density function, respectively. The measures are paired such that the right-side Pareto's $a$ is fixed at a reference value $1.5$ for one measure and varied for the other; we denote the reference value $a_{\text{ref}}$ and the varied value $a_*$. For both measures, we fix the left-side Pareto at $\lambda = 1$ and $a = a_\text{ref}$. We set the right-side Pareto's $\lambda$ and translate it such that the piecewise density is continuous at the seam. Thus, the two measures coincide below the seam, with any differences confined to the tail above it. Figure~\ref{fig:metricres} displays the interaction between the sample budget $N$ and the ability of the energy distance, 2-Wasserstein distance, and Hill estimator \citep{hill1975simple} to identify differences in the low density tail regions of measure pairs over three locations of the seam. Appendix~\ref{app:splice} formalizes and visualizes the piecewise densities described above.

The energy distance becomes more sensitive to the tail index as the sample budget increases. At the smallest budget, its curve is nearly flat for $a_*>1.5$; at the largest budget, its minimum occurs at the known reference index $a_*=1.5$ for all seam locations. The Hill gap shows a similar minimum, whereas the empirical 2-Wasserstein curve remains comparatively flat for $a_*>1.5$. This comparison supports using the energy distance for this testbed and motivates a large training ensemble. Since training at $N=65536$ is computationally infeasible, we use $N=16384$ in Section~\ref{sec:ablation}. Appendix~\ref{app:tailsignal} discusses the moment conditions relevant to the empirical comparison.

\subsection{Ablation grid}
\label{sec:ablation}

We train post-norm transformers with each attention kernel, with and without symlog preprocessing. Figures~\ref{fig:expswap} and~\ref{fig:expinvlin} report the sheared swap and matrix inverse results, respectively. The sheared swap includes mirrored Pareto and Gaussian marginals; the matrix inverse includes only mirrored Pareto marginals because the Gaussian case is unidentifiable (Proposition~\ref{prop:gaussfiber}). Each configuration uses one training seed and is evaluated on $128$ trials. The supplementary pre-norm runs show no ensemble collapse across the tested kernels and transformations (Appendix~\ref{app:pre-norm}). Three observations follow.

\begin{figure}[!htbp]
  \centering
  \includegraphics[width=0.92\linewidth]{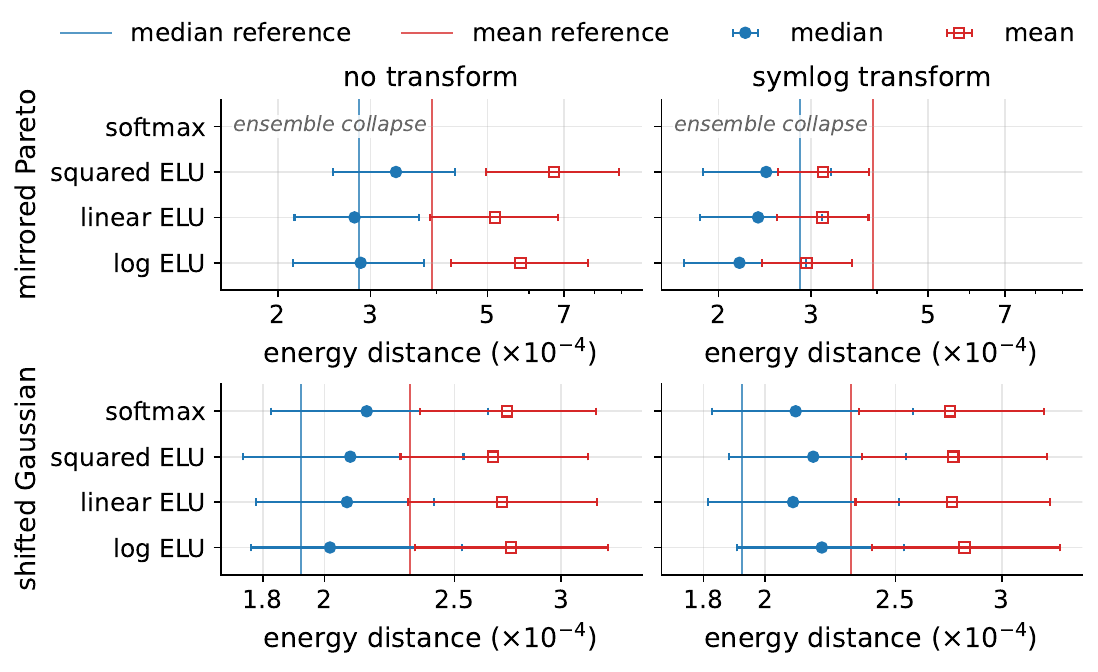}
  \caption{Approximation results for the sheared swap operator using post-norm at $N=16384$. The grid organizes models trained with the symlog transform (right), without the symlog transform (left), on mirrored Pareto marginals (top), and on shifted Gaussian marginals (bottom). Each plot reports a median and mean energy distance for every attention kernel. The vertical lines mark the matching i.i.d.\ sampling reference under the same summary statistic, in the color of that statistic. Bars are $95\%$ bootstrap intervals over evaluation trials (Appendix~\ref{app:bootstrap}). Scores below the sampling reference are possible for model-generated ensembles, as explained in Appendix~\ref{app:ratios}. The label \emph{ensemble collapse} marks runs with degenerate output ensembles; Appendix~\ref{app:collapse} shows an example.}
  \label{fig:expswap}
\end{figure}

\begin{figure}[!htbp]
  \centering
  \includegraphics[width=0.92\linewidth]{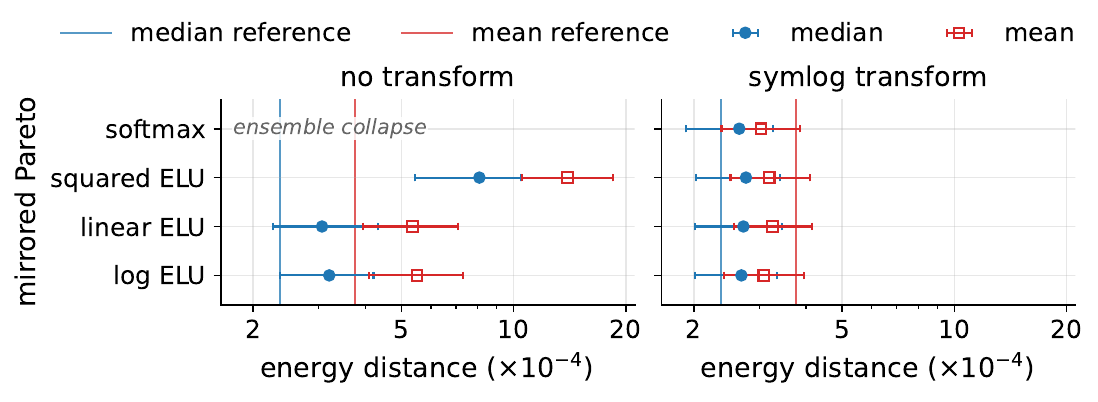}
  \caption{Matrix inverse results on mirrored Pareto marginals, with the same setup as Figure~\ref{fig:expswap}.}
  \label{fig:expinvlin}
\end{figure}

\paragraph{Kernel choice on heavy-tailed inputs.}
Without symlog preprocessing, the post-norm softmax models exhibit ensemble collapse on both heavy-tailed benchmarks. All three slower-growing kernels avoid collapse and achieve lower energy distances. Appendix~\ref{app:collapse} shows the output degeneration and training curves. Relative to the median i.i.d. sampling reference, the median distances of the three slower-growing kernels range from $0.980$ to $1.174$ on the sheared swap task and from $1.295$ to $3.423$ on the matrix inverse task (Tables~\ref{tab:energysummary} and~\ref{tab:energysummaryinvlin}).

\paragraph{Effect of symlog preprocessing.}
With symlog, the softmax model avoids collapse on the matrix inverse task but still collapses on the sheared swap task. For the three slower-growing kernels, symlog reduces the median distance ratios to $0.767$--$0.863$ on the sheared swap task and $1.135$--$1.167$ on the matrix inverse task. The softmax ratio on the matrix inverse task is $1.120$. Symlog compresses Pareto tails to exponential tails (Appendix~\ref{app:logspace}).

\paragraph{Gaussian control.}
On the Gaussian sheared swap benchmark, all four kernels avoid collapse and achieve similar energy distances. The separation between softmax and the slower-growing kernels is thus specific to the heavy-tailed configurations in these experiments. We hypothesize that exponential weighting amplifies the influence of extreme samples and contributes to the observed collapse.

\FloatBarrier
\section{Conclusion}
\label{sec:conclusion}

We design two measure-to-measure operator learning benchmarks with closed-form targets. In the post-norm experiments, softmax exhibits ensemble collapse on both heavy-tailed tasks without preprocessing, while all three slower-growing kernels avoid collapse. Symlog enables softmax to avoid collapse on one of the two tasks. These results support replacing the exponential kernel as an effective design choice for post-norm transformers trained on heavy-tailed ensembles.

Our tail-sensitivity testbed also shows how the sample budget affects empirical distribution comparisons: at larger budgets, the energy distance distinguishes tail-index differences that are less apparent in the empirical 2-Wasserstein curves.

\paragraph{Limitations.}
\label{sec:limitations}
Our experiments use two synthetic operators on two-dimensional measures and one training seed per configuration. The bootstrap intervals describe variation across evaluation instances at fixed trained weights. Training also uses different attention implementations: FlashAttention with bfloat16 autocast for softmax and custom float32 attention for the other kernels (Appendix~\ref{app:compute}). Multiple seeds and comparisons at a common precision would help assess the reproducibility of the observed kernel differences.

\paragraph{Future work.}
\label{sec:future}
Tracking attention weights and gradients around the onset of collapse would help test the explanation proposed in Section~\ref{sec:ablation}. A broader kernel family could also clarify how the shape of the weighting function affects optimization. Appendix~\ref{app:tailstat} complements the energy distance with a Hill-based diagnostic. Evaluating errors against the known tail indices across a range of thresholds would provide a more detailed account of tail approximation and could guide the design of training objectives. Higher-dimensional tasks and nonlinear sample transformations would extend the evaluation to more complex measure-to-measure maps, while applications in data assimilation and inverse problems would test its practical relevance.

\clearpage
\bibliographystyle{plainnat}
\bibliography{references}

\clearpage
\appendix
\FloatBarrier
\raggedbottom
\makeatletter
\let\apx@section\section
\let\apx@subsection\subsection
\renewcommand{\section}{\FloatBarrier\apx@section}
\renewcommand{\subsection}{\FloatBarrier\apx@subsection}
\makeatother

\section{Tail decay and attention integrability}
\label{app:integrability}

The conditions in \eqref{eq:attention-integrability} depend on both the
kernel and the key/value measure. We give a sufficient tail bound for
softmax and verify divergence of its denominator on the mirrored Pareto
family used in this paper.

\begin{definition}
  \label{def:tails}
  A measure $\nu\in\mathcal{P}(\mathbb{R}^{d_v})$ is \emph{light-tailed}
  if $\int e^{t\lVert y\rVert_2}\nu(\mathrm{d}y)<\infty$ for some $t>0$,
  and \emph{heavy-tailed} otherwise.
\end{definition}

Write $S_\nu(r)=\nu(\{y:\lVert y\rVert_2>r\})$. A sufficient condition
for softmax attention to be defined at every query and every finite choice
of weights is the following faster-than-exponential radial tail bound:
\begin{equation}
  \label{eq:tailclass}
  S_\nu(r)\le C\exp\!\bigl(-r^{1+\delta}\bigr)
  \qquad(r\ge r_0)
\end{equation}
for some $\delta>0$, $C>0$ and $r_0\ge0$. This condition is stronger than
light-tailedness. Indeed, for every $t>0$,
\[
  \int e^{t\lVert y\rVert_2}\nu(\mathrm{d}y)
  =1+t\int_0^\infty e^{tr}S_\nu(r)\,\mathrm{d}r<\infty.
\]
For a fixed query $s$, put $b=K^\top Qs/\sqrt{d_K}$. Then the softmax
denominator is positive and bounded by
$\int e^{\lVert b\rVert_2\lVert y\rVert_2}\nu(\mathrm{d}y)$, while its
absolute numerator is bounded by
\[
  \lVert V\rVert_{\mathrm{op}}
  \int \lVert y\rVert_2 e^{\lVert b\rVert_2\lVert y\rVert_2}
       \nu(\mathrm{d}y)
  \le \lVert V\rVert_{\mathrm{op}}
  \int e^{(\lVert b\rVert_2+1)\lVert y\rVert_2}\nu(\mathrm{d}y)<\infty.
\]

\begin{proposition}
  \label{prop:polytail}
  If $S_\nu(r)\ge\gamma r^{-a}$ for all sufficiently large $r$, where
  $a,\gamma>0$, then $\nu$ is heavy-tailed and does not satisfy
  \eqref{eq:tailclass}.
\end{proposition}

\begin{proof}
  For every $t>0$, the tail integral above is bounded below by a positive
  multiple of $\int_{r_0}^\infty e^{tr}r^{-a}\,\mathrm{d}r=\infty$.
  Moreover, $r^{-a}\exp(r^{1+\delta})\to\infty$ for every $\delta>0$,
  which rules out \eqref{eq:tailclass}.
\end{proof}

For the mirrored Pareto variable $Z$ in \eqref{eq:mirrored-sample}, both tails
decay polynomially, so $\mathbb{E}e^{cZ}=\infty$ for every $c\ne0$.
Let $Z_1,Z_2$ be the independent marginals of $\eta$ and let
$\nu=T_\sharp\eta$ for either invertible mixing matrix used in the
benchmarks. If $b\ne0$, then $T^\top b\ne0$, and independence gives
\[
  \mathcal{D}_{\exp}(s;\nu)
  =\mathbb{E}e^{b^\top TZ}
  =\prod_{j=1}^{2}\mathbb{E}e^{(T^\top b)_jZ_j}
  =\infty.
\]
Thus population softmax attention on these measures is undefined for any
nonzero linear score. A zero score instead gives uniform weighting and
requires only integrability of the values.

For a strictly positive kernel satisfying $g(t)\le C_g(1+|t|^q)$
with $q>0$, a finite $(q+1)$st absolute moment of the key/value measure
suffices for both integrals at any fixed query and finite weights. This
follows by bounding the score by a constant times $\lVert y\rVert_2$
and the value by $\lVert V\rVert_{\mathrm{op}}\lVert y\rVert_2$.

\begin{remark}
  \label{rem:integrability}
  The measure $\nu$ in \eqref{eq:attention-integrability} is the input to
  a given attention sublayer; after the first block, it is a learned
  representation rather than the original measure $\mu_I(\omega)$.
\end{remark}

\section{Expectation over the energy distance on sample sets of fixed size}
\label{app:blocking}
For independent i.i.d.\ ensembles from fixed laws with finite absolute
moments at exponent $p$, the empirical energy distance \eqref{eq:energy}
converges to the population formula \eqref{eq:energy-pop} by the law of
large numbers for V-statistics. The training objective
\eqref{eq:population-risk} instead averages a distance computed at fixed
ensemble size. The identity below isolates the difference between
averaging distances over blocks and computing a distance on their
pooled ensembles. The identity holds for fixed ensembles, including model outputs with dependent points.

Fix an exponent $p$ in the range of \eqref{eq:energy-pop}.
All sample points below lie in the same Euclidean space, and the block
count and block size are positive integers.
Fix $N=N_1N_2$ sample pairs, arranged as two ensembles
$Y_1=[y_{1,1},\dots,y_{1,N}]^{\top}$ and $Y_2=[y_{2,1},\dots,y_{2,N}]^{\top}$. Partition the index set $\{1,\dots,N\}$ into $N_1$ blocks
$B_1,\dots,B_{N_1}$ of $N_2$ indices each. For a nonempty set $\mathcal{S}$ of ordered
index pairs write
\begin{equation}
  \mathcal{D}(\mathcal{S})
  \;=\;
  \frac{1}{\lvert\mathcal{S}\rvert}\!\!\sum_{(i,j)\in\mathcal{S}}\!\!
  \Bigl(
     2\lVert y_{1,i}-y_{2,j}\rVert^{p}
     - \lVert y_{1,i}-y_{1,j}\rVert^{p}
     - \lVert y_{2,i}-y_{2,j}\rVert^{p}
  \Bigr),
  \label{eq:contrast}
\end{equation}
the energy contrast averaged over the pairs in $\mathcal{S}$, and let
\begin{equation*}
  \mathcal{A}=\{1,\dots,N\}^{2},
  \qquad
  \mathcal{I}=\bigcup_{b=1}^{N_1} B_b\times B_b,
  \qquad
  \mathcal{O}=\mathcal{A}\setminus\mathcal{I},
\end{equation*}
be all $\lvert\mathcal{A}\rvert=N^{2}$ ordered pairs, the
$\lvert\mathcal{I}\rvert=N_1N_2^{2}=NN_2$ pairs lying inside a common block, and
the $\lvert\mathcal{O}\rvert=N(N-N_2)$ pairs straddling two blocks, respectively.

Applying the empirical energy distance \eqref{eq:energy} to the two ensembles at full size $N$ gives
\begin{equation}
  \mathcal{E}_p^{2}(Y_1,Y_2)
  \;=\;
  \frac{1}{N^{2}}\sum_{i,j=1}^{N}
  \Bigl(
     2\lVert y_{1,i}-y_{2,j}\rVert^{p}
     - \lVert y_{1,i}-y_{1,j}\rVert^{p}
     - \lVert y_{2,i}-y_{2,j}\rVert^{p}
  \Bigr)
  \;=\;
  \mathcal{D}(\mathcal{A}).
  \label{eq:whole}
\end{equation}

For an ensemble $Y$ and an index block $B$, let $Y|_B$ denote the
ensemble restricted to that block. Applying the empirical energy distance
within each block and averaging gives
\begin{equation}
  \overline{\mathcal{E}}_p^{2}
  \;=\;
  \frac{1}{N_1}\sum_{b=1}^{N_1}
  \mathcal{E}_p^{2}\bigl(Y_1|_{B_b},\,Y_2|_{B_b}\bigr)
  \;=\;
  \frac{1}{N_1}\sum_{b=1}^{N_1}\frac{1}{N_2^{2}}\!\!\sum_{i,j\in B_b}\!\!(\cdots)
  \;=\;
  \mathcal{D}(\mathcal{I}).
  \label{eq:blocked}
\end{equation}
Here $(\cdots)$ denotes the energy-contrast summand in \eqref{eq:contrast}.

\begin{proposition}
\label{prop:blocking}
For any fixed $Y_1,Y_2$ and any partition into $N_1$ blocks of size $N_2$
with $\mathcal{O}\ne\varnothing$,
\begin{equation}
  \mathcal{E}_p^{2}(Y_1,Y_2)
  \;=\;
  \frac{1}{N_1}\,\overline{\mathcal{E}}_p^{2}
  \;+\;
  \Bigl(1-\frac{1}{N_1}\Bigr)\,\mathcal{D}(\mathcal{O}),
  \label{eq:blockid}
\end{equation}
equivalently
\begin{equation}
  \mathcal{E}_p^{2}(Y_1,Y_2)-\overline{\mathcal{E}}_p^{2}
  \;=\;
  \Bigl(1-\frac{1}{N_1}\Bigr)
  \bigl[\mathcal{D}(\mathcal{O})-\mathcal{D}(\mathcal{I})\bigr].
  \label{eq:blockdiff}
\end{equation}
If $\mathcal{O}=\varnothing$, the full and blockwise distances coincide;
no off-block contrast is defined or needed.
\end{proposition}

\begin{proof}
If $\mathcal{O}$ is empty, then $\mathcal{A}=\mathcal{I}$, and
\eqref{eq:whole} and \eqref{eq:blocked} give equality directly.
Otherwise, both parts of the disjoint union
$\mathcal{A}=\mathcal{I}\cup\mathcal{O}$ are nonempty. Since
$\mathcal{D}$ averages the same energy-contrast summand over each index
set, splitting the full sum yields
$N^{2}\mathcal{D}(\mathcal{A})
=\lvert\mathcal{I}\rvert\mathcal{D}(\mathcal{I})
+\lvert\mathcal{O}\rvert\mathcal{D}(\mathcal{O})$.
Dividing by $N^{2}$ and substituting
$\lvert\mathcal{I}\rvert/N^{2}=1/N_1$ and
$\lvert\mathcal{O}\rvert/N^{2}=1-1/N_1$ gives \eqref{eq:blockid} after
identifying $\mathcal{D}(\mathcal{I})=\overline{\mathcal{E}}_p^{2}$ by
\eqref{eq:blocked}; subtracting $\overline{\mathcal{E}}_p^{2}$ from both sides
gives \eqref{eq:blockdiff}.
\end{proof}

Thus averaging blockwise distances can differ from computing the distance on the pooled ensembles, as quantified by \eqref{eq:blockdiff}.

\section{The marginal families}
\label{app:distributions}

This appendix section complements Section \ref{sec:heavylightmarginals}.

\subsection{The mirrored Pareto family and the light-tailed control}
\label{app:mirrored}

We now restate the Pareto and mirrored Pareto definitions and visualize their densities. A Pareto distribution is parameterized by a scale $\lambda>0$ and tail index $a>0$. It is supported on
$[\lambda,\infty)$ with survival function $S_P(z)=(z/\lambda)^{-a}$ and density
$f_P(z)=a\lambda^{a}z^{-a-1}$. The Pareto density is discontinuous at $\lambda$.

Reflecting the Pareto about the origin and translating the reflected branch by $2\lambda$ makes the two branches meet continuously at $z=\lambda$, the \emph{seam}.

In this construction and the moment calculation below, $X$ denotes the
underlying Pareto variable and $Z$ its mirrored version.
Let $X\sim\mathrm{Pareto}(\lambda,a)$ and let
$\mathrm{Rad}$ be an independent Rademacher sign, equally likely to be
positive or negative. Then
\begin{equation}
  Z \;=\; \begin{cases} X, & \mathrm{Rad}=+1,\\ 2\lambda - X, & \mathrm{Rad}=-1,\end{cases}
  \label{eq:mirrored-sample}
\end{equation}
which is supported on all of $\mathbb{R}$. Writing the \emph{folded}
coordinate $m(z)=\lambda+|z-\lambda|$, the density and cumulative distribution function are
\begin{equation}
  f(z) \;=\; \tfrac{1}{2} f_P\big(m(z)\big),
  \qquad
  F(z) \;=\;
  \begin{cases}
    \tfrac{1}{2} S_P\big(m(z)\big), & z<\lambda,\\[2pt]
    1 - \tfrac{1}{2} S_P\big(m(z)\big), & z\ge\lambda.
  \end{cases}
  \label{eq:mirrored-density}
\end{equation}
The translation is what makes $f$ continuous: at $z=\lambda$ the right branch
contributes $\tfrac12 f_P(\lambda)$ and the left branch
$\tfrac12 f_P(2\lambda-z)|_{z=\lambda}$, which are equal.

The construction also preserves quantile transport, the increasing map
that matches cumulative probabilities between source and destination
laws. Let $(\lambda_{\mathrm{src}},a_{\mathrm{src}})$ and
$(\lambda_{\mathrm{dst}},a_{\mathrm{dst}})$ be their scale and index pairs.
The corresponding Pareto transport and source folded coordinate are
\[
  T_P(x)=\lambda_{\mathrm{dst}}
    (x/\lambda_{\mathrm{src}})^{a_{\mathrm{src}}/a_{\mathrm{dst}}},
  \qquad
  m_{\mathrm{src}}(z)=\lambda_{\mathrm{src}}+
    |z-\lambda_{\mathrm{src}}|.
\]
The transport between the mirrored laws is therefore
\begin{equation}
  T(z) \;=\;
  \begin{cases}
    T_P\big(m_{\mathrm{src}}(z)\big), & z \ge \lambda_{\mathrm{src}},\\
    2\lambda_{\mathrm{dst}} - T_P\big(m_{\mathrm{src}}(z)\big), & z < \lambda_{\mathrm{src}}.
  \end{cases}
  \label{eq:mirrored-transport}
\end{equation}
Quantile transport allows measures to be sampled in closed form. Figure~\ref{fig:families} compares the densities of the Pareto distribution, mirrored Pareto distribution, and Gaussian distribution.

\begin{figure}[H]
  \centering
  \includegraphics[width=\linewidth]{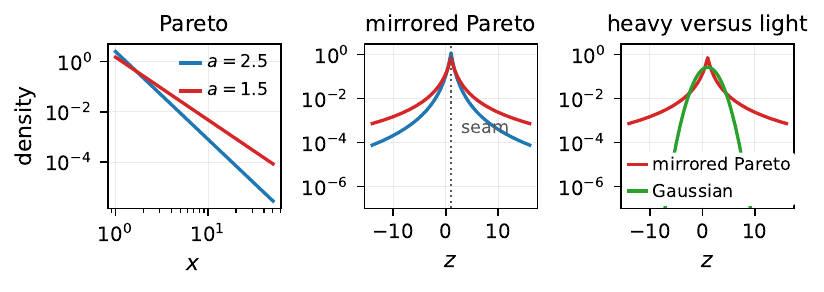}
  \caption{The Pareto, mirrored Pareto, and Gaussian densities centered at 1. Left: the Pareto
  density \eqref{eq:Pareto} on
  log--log axes for two tail indices. Center: the mirrored Pareto density \eqref{eq:mirrored} at the same two tail
  indices, with the seam $z=\lambda$ dotted, only the vertical axis is plotted on a log scale. Right: the mirrored Pareto density
  against a shifted Gaussian density, only the vertical axis is plotted on a log scale.}
  \label{fig:families}
\end{figure}

\subsection{The admissible range of tail indices}
\label{app:momentcondition}
Proposition~\ref{prop:paretomoments} establishes the moment conditions used in Section~\ref{sec:heavylightmarginals}: for the mirrored Pareto distribution, the interval $a \in [1 + \epsilon, 2]$ gives a finite first
moment and an infinite variance.

\begin{proposition}
  \label{prop:paretomoments}
  Let $Z$ be mirrored Pareto with scale $\lambda>0$ and index $a>0$, as in
  \eqref{eq:mirrored-sample}, and let $p$ be positive.
  Then $\mathbb{E}\lvert Z\rvert^{p}<\infty$ if and
  only if $p<a$. In particular $Z$ has a finite mean exactly when $a>1$ and a finite variance
  exactly when $a>2$.
\end{proposition}

\begin{proof}
  First compute the positive moment of the underlying
  $X\sim\mathrm{Pareto}(\lambda,a)$:
  \[
    \mathbb{E}X^{p}
      = \int_{\lambda}^{\infty}x^{p}\,a\lambda^{a}x^{-a-1}\,\mathrm{d}x
      = a\lambda^{a}\int_{\lambda}^{\infty}x^{p-a-1}\,\mathrm{d}x ,
  \]
  which converges exactly when $p-a-1<-1$, i.e.\ when $p<a$, and equals
  $a\lambda^{p}/(a-p)$ there. At $p=a$ the divergence is logarithmic;
  above the threshold it is a power-law divergence.

  Next transfer this criterion to the mirrored variable. The construction
  \eqref{eq:mirrored-sample} gives $m(Z)=X$ on either sign branch, so
  the folded coordinate $m(Z)=\lambda+\lvert Z-\lambda\rvert$ is itself
  $\mathrm{Pareto}(\lambda,a)$. Since $\lambda>0$ we have
  $\lvert z\rvert\le m(z)$ for every $z$, with equality for $z\ge\lambda$, so
  $\mathbb{E}\lvert Z\rvert^{p}\le\mathbb{E}\,m(Z)^{p}<\infty$ whenever $p<a$.
  Conversely $m(z)\le2\lambda+\lvert z\rvert$, and
  $(u+v)^{p}\le C_p(u^{p}+v^{p})$ for $u,v\ge0$ with $C_p=\max(1,2^{p-1})$, so a
  finite $\mathbb{E}\lvert Z\rvert^{p}$ forces a finite
  $\mathbb{E}\,m(Z)^{p}$ and hence $p<a$.
  The mean and variance assertions follow by taking the corresponding
  moment orders. A finite mean here is a Lebesgue expectation, not a
  symmetric principal value.
\end{proof}

\paragraph{Averaging over the parameter distribution.}
The positive cutoff $\epsilon$ also controls the first moment after
averaging over $\Pi_\eta$. For every $a\ge1+\epsilon$, the preceding proof
gives
\[
  \mathbb{E}(|Z|\mid a)
  \le\lambda\frac{a}{a-1}
  \le\lambda\left(1+\frac{1}{\epsilon}\right).
\]
The target transformations $MP$ and $B^{-1}$ have uniformly bounded
operator norms over the parameter ranges in
\eqref{eq:shearparams} and \eqref{eq:iwasawaparams}. Consequently,
\[
  \mathbb{E}_{\omega\sim\Pi}
  \int\lVert y\rVert_2\,\mu_O(\omega)(\mathrm{d}y)<\infty.
\]
For fixed model parameters, finiteness of the expected training loss
\eqref{eq:population-risk} then follows whenever
$\mathbb{E}_{\omega,X}N^{-1}\sum_{i=1}^N
\lVert\widetilde{\Psi}_\theta(X)_i\rVert_2<\infty$.
This condition holds for the post-norm model described in Appendix~\ref{app:arch}.

\subsection{The tail of a Pareto under the signed-log transform}
\label{app:logspace}

In Section \ref{sec:ablation} we ablate preprocessing with the symlog transform
\begin{equation}
  \sigma(x) \;=\; \operatorname{sign}(x)\,\log(1+|x|),
  \qquad
  \sigma^{-1}(y) \;=\; \operatorname{sign}(y)\,(e^{|y|}-1).
  \label{eq:signedlog}
\end{equation}
The transform in \eqref{eq:signedlog} preserves sign, is close to the identity near zero, and compresses large magnitudes like a logarithm.

If $X\sim\mathrm{Pareto}(\lambda,a)$, then for $t\geq\log(1+\lambda)$,
\begin{equation}
  \mathbb{P}(\sigma(X)>t)=\lambda^{a}(e^{t}-1)^{-a}=\lambda^{a}e^{-at}(1+O(e^{-t})),
  \label{eq:logPareto}
\end{equation}
where the last equality is asymptotic as $t\to\infty$. Thus the signed-log transform converts the polynomial tail decay into an exponential tail decay of rate $a$. The transform does not create faster-than-exponential decay and hence does not by itself guarantee finiteness of the population softmax numerator and denominator over unbounded learned score scales.

\section{The operators and their parameter distributions}
\label{app:operators}

This appendix section complements Section \ref{sec:operators}.
\subsection{The mixing matrices and the parameter distributions \texorpdfstring{$\Pi_T$}{Pi T}}
\label{app:shearparams}

Section~\ref{sec:learningproblem} writes every measure in the study as
$\mu(\omega)=T(\omega_T)_\sharp\eta(\omega_\eta)$, with $T$ a linear transformation
drawn through $\omega_T\sim\Pi_T$ and $\eta$ a product of the one-dimensional
marginals of Appendix~\ref{app:distributions}. The two operator benchmarks of
Section~\ref{sec:operators} differ in the linear transformation $T$ on the input and output measures. For both operators, $\omega_T\sim\Pi_T$ is drawn once per parameter instance
and shared by that instance's input and output measure. A fresh $\omega_T$ is drawn for every trial in a training step.

\paragraph{Sheared swap.}
The latent target is $P_\sharp\eta$ with $P=\left(
\begin{smallmatrix} 0 & 1 \\ 1 & 0 \end{smallmatrix}\right)$ the coordinate
exchange, and both input measure and output measure are observed through $M$, a product of shear matrices.
Writing $\omega_T=(j,k)$,
\begin{equation}
  M(\omega_T)
    = \begin{pmatrix} 1 & j\\ 0 & 1\end{pmatrix}
      \begin{pmatrix} 1 & 0\\ k & 1\end{pmatrix}
    = \begin{pmatrix} 1+jk & j\\ k & 1\end{pmatrix},
  \qquad
  \Pi_T = \mathcal{U}\bigl([-1,1]^{2}\bigr),
  \label{eq:shearparams}
\end{equation}
so $j$ and $k$ are independent and $\det M(\omega_T)=1$ identically. The model
receives input samples from $\mu_I(\omega) = M_\sharp\eta$ and must produce samples from $\mu_O(\omega) = (MPM^{-1})_\sharp(M_\sharp\eta)$.

\paragraph{Matrix inverse.}
$B$ is a general element of $GL^{+}(2)$, the invertible real matrices
on the plane with positive determinant, parameterized as the product of a rotation, a unit upper-triangular shear, and a positive diagonal. With
$\omega_T=(\vartheta,s,\lambda_1,\lambda_2)$,
\begin{equation}
  B(\omega_T)
    = \underbrace{\begin{pmatrix}
        \cos\vartheta & -\sin\vartheta\\
        \sin\vartheta & \cos\vartheta\end{pmatrix}}_{Q(\vartheta)}
      \underbrace{\begin{pmatrix} 1 & s\\ 0 & 1\end{pmatrix}}_{U(s)}
      \underbrace{\mathrm{diag}(\lambda_1,\lambda_2)}_{D(\lambda)} ,
  \label{eq:iwasawa}
\end{equation}
with the four parameters drawn independently,
\begin{equation}
  \Pi_T = \mathcal{U}([-0.7,\,0.7])
          \otimes\mathcal{U}([-0.35,\,0.35])
          \otimes\mathcal{U}([0.75,\,1.5])^{\otimes2} .
  \label{eq:iwasawaparams}
\end{equation}
The model receives input samples from $\mu_I(\omega) = B_\sharp\eta$ and must produce samples from $\mu_O(\omega) =(B^{-2})_\sharp(B_\sharp\eta)$.

\subsection{Choice of operators}
\label{app:opchoice}

Three considerations motivate these operators.
\paragraph{Linearity/nonlinearity.} Both operators depend nonlinearly on their defining matrices, $M$ for the sheared swap and $B$ for the matrix inverse. Nonlinear dependence on $\eta$ is left for future work. 
\paragraph{Interpretability.} Working in 2D allows us to visualize the transport of samples directly. Figure \ref{fig:transport} displays the transport of samples from the model using post-norm, linear-ELU, and no symlog transformation, for the sheared swap task where $j=k=0$, on heavy-tailed measures.

\begin{figure}[t]
  \centering
  \includegraphics[width=0.78\linewidth]{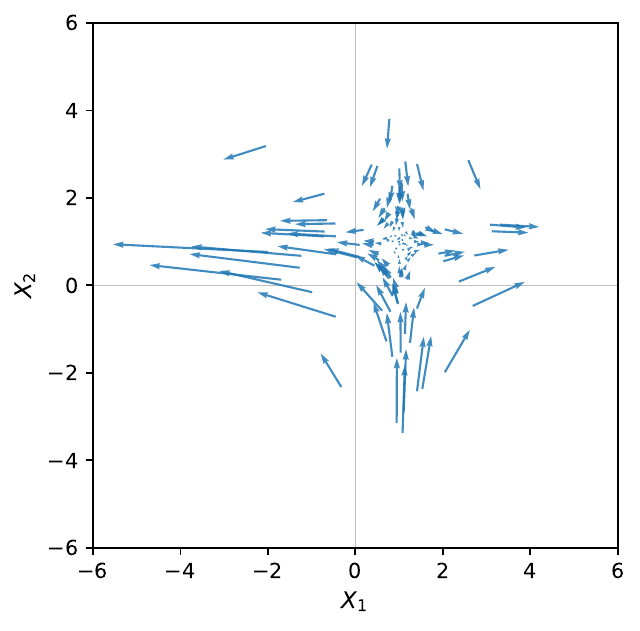}
  \caption{The transport a trained model realizes. Arrows run from each input point to the model's output, for the
  linear-ELU post-norm model evaluated at $j=k=0$, where the shear is the
  identity and the sheared swap reduces to a coordinate exchange.
  $\TransportArrows$ arrows are drawn from a set of $\TransportSamples$ points,
  sampled uniformly from the $\TransportWindowMass\%$ of the set that falls
  inside the window; for this parameter draw the two marginals carry tail
  indices $a_1=\TransportAlphaOne$ and $a_2=\TransportAlphaTwo$.}
  \label{fig:transport}
\end{figure}

\paragraph{Identifiability.} We require that the output measure is determined by the input measure for each operator, which we call identifiability. Among the four combinations considered, the matrix inverse on shifted Gaussian marginals is unidentifiable and is excluded from the learning experiments. Appendix~\ref{app:gaussid} establishes the three identifiable cases and gives the Gaussian counterexample.

\subsection{Identifiability}
\label{app:gaussid}

The model observes samples from the input measure $\mu_I(\omega)$ only, and must
produce samples from the output measure $\mu_O(\omega)$. If two parameter
instances share an input measure but not an output measure, then no model can be
exact on both.

\begin{definition}
  \label{def:identifiable}
  An operator $\Psi$ and a family of measures $\mathcal{M}(\mu,\Pi)$ are
  \emph{identifiable} when, for every pair of parameter instances
  $\omega,\omega'$ in the support of $\Pi$, including its boundary,
  \begin{equation*}
    \mu_I(\omega)=\mu_I(\omega')
    \quad\Longrightarrow\quad
    \mu_O(\omega)=\mu_O(\omega') .
  \end{equation*}
  Otherwise they are \emph{unidentifiable}.
\end{definition}

Both operators of Section~\ref{sec:operators} write the output measure as a
pushforward of the input measure by a matrix determined by $\omega_T$, namely
$MPM^{-1}$ for the sheared swap and $B^{-2}$ for the matrix inverse. Recovering
the matrix $T\in\{M,B\}$ from the input measure alone is therefore sufficient for
identifiability. Exhibiting two parameter instances that share an input measure
while their matrices give different output measures is sufficient for
unidentifiability. Throughout this subsection,
$Z=(Z_1,Z_2)^\top\sim\eta(\omega_\eta)$ has independent
coordinates drawn through $\Pi_\eta$ of Section~\ref{sec:heavylightmarginals}. The subscripts
$\ell,h$ label those coordinates by parameter interval rather than by position:
$Z_\ell$ is drawn from the lower interval and $Z_h$ from the higher one,
whether the parameter is a Pareto index or a Gaussian standard deviation.
Section~\ref{sec:heavylightmarginals} exchanges the two ranges between the
families so that $Z_2$ always carries the greater spread, so the assignment is
$(Z_1,Z_2)=(Z_h,Z_\ell)$ for the mirrored Pareto and $(Z_1,Z_2)=(Z_\ell,Z_h)$
for the shifted Gaussian. Let
$X=TZ$, so $X\sim\mu_I(\omega)=T_\sharp\eta$. Our two operators and two marginal families
make four combinations. Table~\ref{tab:identifiability} states the outcome of
each combination and the proposition that decides it.

\begin{table}[H]
  \centering
  \caption{Identifiability of each operator paired with each marginal family,
  in the every-pair sense of Definition~\ref{def:identifiable}.}
  \label{tab:identifiability}
  \small
  \begin{tabular}{lccl}
    \toprule
    marginal family & sheared swap $M$ & matrix inverse $B$ & decided by \\
    \midrule
    mirrored Pareto  & yes & yes & Propositions~\ref{prop:meanid} and~\ref{prop:tailid} \\
    shifted Gaussian & yes & no  & Propositions~\ref{prop:meanid} and~\ref{prop:gaussfiber} \\
    \bottomrule
  \end{tabular}
\end{table}

Central symmetry about a point means invariance of the law under
reflection through that point. Both latent families are centrally symmetric about
$z_*=(\lambda,\lambda)^\top$, since their fixed locations agree.
Consequently the input measure is centrally symmetric about $c=Tz_*$,
including at Pareto indices where a mean does not exist. This center is
unique: distinct centers would, by composing their reflections, make a
probability measure invariant under a nonzero translation. A sufficiently
small ball of positive mass would then have infinitely many disjoint
translates of the same mass, a contradiction. Thus $c$ is determined by
the input measure alone.

When the first absolute moments are finite, symmetry also identifies the
center with the mean. Since $\lambda$ and
$\mu_{\mathcal{N}}$ are both fixed at $1$,
\begin{equation}
  \mathbb{E}[Z_i]=1,
  \qquad\text{hence}\qquad
  \mathbb{E}[X]=T\,\mathbb{E}[Z]=T\mathbf{1},
  \qquad \mathbf{1}=(1,1)^{\top},
  \label{eq:tailmean}
\end{equation}
for either family, with $i\in\{\ell,h\}$, whenever these expectations exist. The arguments below
use the center rather than assuming existence of these expectations.

\begin{proposition}
  \label{prop:meanid}
  The sheared swap is identifiable on the mirrored Pareto family and on the
  shifted Gaussian family.
\end{proposition}

\begin{proof}
  Write $c=(c_x,c_y)^\top$ for the unique input center and recall that
  the common latent location $\lambda$ is known and positive. Multiplying $M$ in
  \eqref{eq:shearparams} by $z_*$ determines $k$ from the vertical
  coordinate of $c$ and, when that coordinate is nonzero, determines $j$:
  \begin{equation}
    k = \frac{c_y-\lambda}{\lambda},
    \qquad
    j = \frac{c_x-\lambda}{c_y}.
    \label{eq:meanid}
  \end{equation}
  It remains to treat the shear edge on which $c_y$ vanishes. The
  formula for $k$ still applies there and fixes it at its lower allowed
  endpoint; only $j$ remains unknown.
  Write $X_x,X_y$ for the observed coordinates. On this edge,
  \[
    X_y=Z_2-Z_1,\qquad X_x=Z_1+jX_y .
  \]

  For mirrored Pareto marginals, let $f_1,f_2$ be the densities of $Z_1,Z_2$.
  The inverse coordinate change has unit Jacobian determinant, so
  independence gives the continuous input density
  \[
    (x,y)\longmapsto f_1(x-jy)\,f_2(x-jy+y).
  \]
  Each factor is smooth away from its seam and has distinct, nonzero
  one-sided derivatives there. Thus the density is smooth off the lines
  \[
    x-jy=\lambda,\qquad x-jy+y=\lambda,
  \]
  and is not differentiable on either line away from their intersection:
  the other factor is smooth and strictly positive, so it cannot remove
  the derivative jump. Equality of measures implies equality everywhere
  of these continuous densities, hence determines this unordered pair
  of lines. Their slopes, with $x$ viewed as a function of $y$, are
  distinct, and the larger slope is $j$. This recovers $j$ for every
  allowed pair of Pareto indices.

  For Gaussian marginals, let $v,w$ be the variances of $Z_1,Z_2$,
  and let $C=\operatorname{Cov}(X)$, with $C_{xy}$ its
  off-diagonal entry. The shifted Gaussian gives its second coordinate the
  higher standard deviation interval, so $v\le w$. Independence and the unit
  determinant of $M$ give
  \[
    S:=\operatorname{Var}(X_y)=v+w,\qquad
    \det C=vw,\qquad C_{xy}=jS-v.
  \]
  The sum and product determine the unordered positive roots of
  $t(S-t)=\det C$; ordering determines $v,w$, also when they coincide.
  Since $S$ is positive, $j=(C_{xy}+v)/S$ is determined.

  Hence the input determines $M$ on the edge as well as off it, for
  either family. The pushforward
  $\mu_O=(MPM^{-1})_\sharp\mu_I$ then determines the output for every
  parameter instance.
\end{proof}

The common nonzero latent location supplies the center information used
in \eqref{eq:meanid}. On the exceptional shear edge, the Pareto density
or the Gaussian covariance supplies the remaining information.

\begin{proposition}
  \label{prop:tailid}
  The matrix inverse is identifiable on the mirrored Pareto family.
\end{proposition}

\begin{proof}
  Write $a_\ell,a_h$ for the indices drawn from the lower and higher
  intervals and $b_\ell,b_h$ for the columns of $B$ that multiply
  $Z_\ell,Z_h$; on the mirrored Pareto family these are the second and
  first columns of $B$ respectively. Let $c$ be the unique
  input center. Then
  \[
    m:=c/\lambda=b_\ell+b_h .
  \]
  The ranges in Section \ref{sec:heavylightmarginals} imply $a_\ell\le a_h$, with equality
  possible only at their common endpoint.

  First justify the directional tail calculation. For a scalar mirrored
  variable $W$ of index $a$, folding gives the centered absolute tail
  \[
    \mathbb{P}(|W-\lambda|>r)
      =\bigl(\lambda/(\lambda+r)\bigr)^a
      \sim\lambda^a r^{-a}
      \qquad (r\to\infty).
  \]
  Here $\sim$ denotes asymptotic equivalence. Scaling $W-\lambda$
  by a nonzero scalar multiplies the tail constant by the
  absolute scalar raised to $a$.
  For independent scaled centered coordinates $U,V$, fix
  $0<\gamma<1$.
  The disjoint events on which one summand exceeds $r+\gamma r$
  in absolute value and the other is at most $\gamma r$ give
  \[
  \begin{aligned}
    \mathbb{P}(|U+V|>r)
    &\ge \mathbb{P}(|U|>r+\gamma r)\mathbb{P}(|V|\le\gamma r)\\
    &\quad+\mathbb{P}(|V|>r+\gamma r)\mathbb{P}(|U|\le\gamma r).
  \end{aligned}
  \]
  Conversely, if the sum exceeds $r$ and the summands do not both
  exceed $\gamma r$, one must exceed $r-\gamma r$. Therefore
  \[
  \begin{aligned}
    \mathbb{P}(|U+V|>r)
    &\le \mathbb{P}(|U|>r-\gamma r)
         +\mathbb{P}(|V|>r-\gamma r)\\
    &\quad+\mathbb{P}(|U|>\gamma r)
             \mathbb{P}(|V|>\gamma r).
  \end{aligned}
  \]
  After multiplication by $r$ raised to the smaller index, the product
  term vanishes: its decay exponent is the sum of the positive indices,
  by independence. The lower and upper bounds then squeeze the
  rescaled tail as $r$ tends to infinity and $\gamma$ subsequently
  tends to zero. Thus the leading tail constants add for equal indices;
  only the heavier tail remains
  for unequal indices. An absent summand is simply omitted.

  For a vector $u$ in the observation plane, use
  $U=\langle u,b_\ell\rangle(Z_\ell-\lambda)$ and
  $V=\langle u,b_h\rangle(Z_h-\lambda)$. Their sum is
  $\langle u,X-c\rangle$, so the calculation yields
  \begin{equation}
    F(u):=\lim_{r\to\infty}r^{a_\ell}
       \mathbb{P}\bigl(|\langle u,X-c\rangle|>r\bigr)
    =
    \begin{cases}
      \lambda^{a_\ell}|\langle u,b_\ell\rangle|^{a_\ell},
        &a_\ell<a_h,\\
      \lambda^{a_\ell}u^\top BB^\top u,
        &a_\ell=a_h .
    \end{cases}
    \label{eq:tailatom}
  \end{equation}
  The equal-index expression is quadratic because the common index is
  $2$. Each nonzero projection has an absolute tail with a power-law
  exponent. The smallest such exponent is $a_\ell$, because some
  direction has a nonzero coefficient along $b_\ell$. This identifies
  $a_\ell$ from the input law without knowing the columns in advance. Hence
  $F$ is also known. In the unequal-index case $F$ vanishes on a line,
  whereas in the equal-index case it is positive in every nonzero
  direction because $B$ is invertible. The input therefore distinguishes
  these cases, even when the smaller index is the common endpoint.

  If $a_\ell<a_h$, the first line of \eqref{eq:tailatom} determines
  $b_\ell$ up to sign: the zero set of $F$ gives the perpendicular line and
  its values fix the magnitude. For either sign, the other column must
  be $b_h=m-b_\ell$, so each candidate matrix has determinant
  $\pm\det(b_\ell,m)$, the outer sign fixed by the family's coordinate
  order and the inner one by the sign choice. The two candidates therefore
  have determinants of opposite sign. They are nonzero because the actual
  $B$ is invertible.
  Positive determinant selects exactly one candidate.

  If $a_\ell=a_h$, the quadratic form determines $G=BB^\top$.
  Put $d=b_\ell-b_h$. Expanding the outer products gives
  \[
    dd^\top=2G-mm^\top .
  \]
  Invertibility excludes equal columns, so $d$ is nonzero.
  Thus this rank-one matrix determines $d$ up to sign. For each
  sign, the linear equations
  \[
    b_\ell+b_h=m,\qquad b_\ell-b_h=d
  \]
  determine the columns uniquely. Changing the sign exchanges the
  columns and reverses the determinant. Again, positive determinant
  selects exactly one candidate.

  Thus in every case the input law determines the full matrix $B$,
  without a first-moment assumption. Applying the inverse matrix to the
  input recovers the latent law $\eta=(B^{-1})_\sharp\mu_I$; applying
  it again gives the prescribed output $\mu_O=(B^{-1})_\sharp\eta$.
\end{proof}

\begin{proposition}
  \label{prop:gaussfiber}
  The matrix inverse is unidentifiable on the shifted Gaussian family.
\end{proposition}

\begin{proof}
  Choose $B=\operatorname{diag}(b_x,b_y)$ with distinct scales $b_x,b_y$ strictly inside
  the scale interval in \eqref{eq:iwasawaparams}, and choose each latent
  standard deviation strictly inside its interval in
  Section~\ref{sec:heavylightmarginals}.
  This is an interior parameter instance: the rotation and shear are
  at their central values.
  Write $D=\mathrm{diag}(\sigma_\ell^{2},\sigma_h^{2})$ for the latent
  covariance, in the shifted Gaussian's coordinate order $(Z_1,Z_2)=(Z_\ell,Z_h)$, so that
  $X\sim\mathcal{N}(B\mathbf{1},\,BDB^{\top})$ because the latent mean
  is $\mathbf{1}$. For a sufficiently small real perturbation $h$, define
  \begin{equation}
    H = \begin{pmatrix} 1+h & -h \\ g & 1-g \end{pmatrix},
    \qquad
    g = \frac{h\,\sigma_h^{2}}
             {(1+h)\sigma_\ell^{2} + h\,\sigma_h^{2}} .
    \label{eq:invlinfiber}
  \end{equation}
  The denominator defining $g$ is positive when $h$ vanishes, where
  $H=I$, with $I$ the identity matrix.
  For sufficiently small $h$, it stays positive and $H$ stays invertible
  with positive determinant. The row sums give
  $H\mathbf{1}=\mathbf{1}$, so $H$ and its inverse preserve the
  common latent location. The off-diagonal entry of $HDH^\top$ is linear in $g$;
  setting it to zero gives exactly the expression in \eqref{eq:invlinfiber}.
  Thus $HDH^\top$ is diagonal; its entries are positive by positive
  definiteness. Put $B'=BH^{-1}$ and
  $D'=\mathrm{diag}(\sigma_\ell'^{2},\sigma_h'^{2})=HDH^{\top}$. Then
  $B'\mathbf{1}=B\mathbf{1}$ and $B'D'B'^{\top}=BDB^{\top}$, so the two parameter
  instances share an input measure, since a Gaussian law is determined
  by its mean and covariance.

  To check that the perturbed instance remains in the prescribed family,
  let $h$ tend to zero. Then $B'$ approaches the chosen positive
  diagonal $B$, and the positive square roots of the diagonal entries
  of $D'$ approach the chosen latent standard deviations.
  Gram--Schmidt applied to the columns of $B'$ gives continuous
  coordinates near $B$: the first column has positive length, the
  determinant stays positive, and the rotation angle has a continuous
  branch near the central value. Because every original parameter lies
  strictly inside its allowed interval, all perturbed parameters remain
  admissible for sufficiently small $h$.

  Finally, the output means are the images of $z_*$ under the respective
  inverse matrices. Since $B'^{-1}=HB^{-1}$, their difference is
  $(H-I)B^{-1}z_*$. Its horizontal component is
  \[
    \lambda h\,\frac{b_y-b_x}{b_xb_y},
  \]
  which is nonzero for every sufficiently small nonzero $h$, since
  $b_x\ne b_y$. Such an $h$ therefore gives admissible instances with
  identical input laws and distinct output means, proving
  unidentifiability.
\end{proof}

The counting behind Proposition~\ref{prop:gaussfiber} is that a
two-dimensional Gaussian exposes five numbers, two in its mean and three in its
covariance, while the matrix inverse carries four parameters in $B$ and two in
$\sigma$. Proposition~\ref{prop:gaussfiber} establishes unidentifiability by constructing two parameter instances with the same input measure and different output measures. For the mirrored Pareto family,
\eqref{eq:tailatom} supplies either a latent axis or, at the equal-index
endpoint, a quadratic tail form. Together with the center and positive
determinant, either determines the full matrix.

\section{Model and compute}
\label{app:model}

\subsection{The model in full}
\label{app:arch}

The model $\widetilde{\Psi}_\theta:\mathbb{R}^{N\times2}\to\mathbb{R}^{N\times2}$ in Section~\ref{sec:setting} is a six-layer non-causal transformer. Writing
$d=2$ for the state dimension and $D=256$ for the embedding dimension, the forward
pass is
\begin{equation}
  X \;\xrightarrow{\;\sigma\;}\;
  \mathbb{R}^{N\times d}
  \;\xrightarrow{\;W_{\mathrm{in}}\;}\;
  \mathbb{R}^{N\times D}
  \;\xrightarrow{\;\text{6 blocks}\;}\;
  \mathbb{R}^{N\times D}
  \;\xrightarrow{\;W_{\mathrm{out}}\;}\;
  \mathbb{R}^{N\times d}
  \;\xrightarrow{\;\sigma^{-1}\;}\;
  \widehat{Y},
  \label{eq:flow}
\end{equation}
where $W_{\mathrm{in}}$ and $W_{\mathrm{out}}$ are untied affine input and output maps, and $\sigma$ is the symlog transform
\eqref{eq:signed-log-main} when used and the identity
otherwise.

\paragraph{Block structure and residual scaling.}
Each block runs an attention sublayer and a feed-forward sublayer, both wrapped
in residual connections, with a single head, inner width $512$ and a GELU
nonlinearity. Under the post-norm placement of Section~\ref{sec:norm} a sublayer
$f$ computes $x\mapsto\mathrm{LN}(\alpha_{\mathrm{RS}}x+f(x))$; under pre-norm it
computes $x\mapsto x+f(\mathrm{LN}(x))$. The scalar
$\alpha_{\mathrm{RS}}=4.899$ is inspired by DeepNorm residual scaling \citep{wang2022deepnet}.

For fixed parameters, the final LayerNorm in the post-norm model bounds its outputs, including after the affine read-out and inverse symlog. The bound depends on the learned parameters, so the output range can adapt during training to approximate the target tails over a finite range. The Hill diagnostic in Appendix~\ref{app:tailstat} assesses this finite-sample tail approximation.

Table~\ref{tab:hyperparams} lists the shared model and optimizer settings, and Appendix~\ref{app:compute} reports the training costs.

\subsection{The attention kernel suite}
\label{app:kernelviz}

Table~\ref{tab:family} defines the four maps $g$ that replace the exponential in
\eqref{eq:kernel-attention}, and Figure~\ref{fig:kernels} shows them.

\begin{figure}[H]
  \centering
  \includegraphics[width=\linewidth]{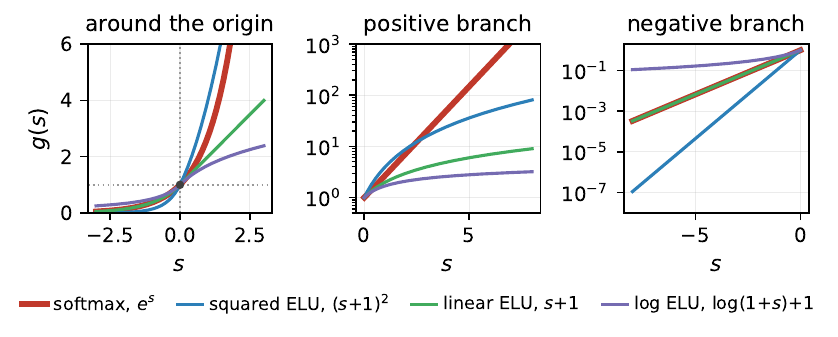}
  \caption{The attention kernel suite of Table~\ref{tab:family}. Left: the four
  maps on a linear vertical axis over the scores where they differ most
  visibly, with the common value $g(0)=1$ marked. Center: the same maps on a
  logarithmic vertical axis over a wider positive range, where the ordering of
  the last column of Table~\ref{tab:family} --- exponential, quadratic, linear,
  logarithmic --- appears as a difference in slope. Right: The same maps over a wider negative range. Every map is strictly positive, so the row normalizer in
  \eqref{eq:kernel-attention} cannot vanish.}
  \label{fig:kernels}
\end{figure}

The negative branches in Table~\ref{tab:family} match the value and slope of the positive branches at the origin, so each map is $C^1$ with $g(0)=1$.

\subsection{Hyperparameters}
\label{app:hyperparams}

Table~\ref{tab:hyperparams} lists the configuration shared across the models in Figures~\ref{fig:expswap} and~\ref{fig:expinvlin}. 

\begin{table}[H]
  \centering
  \caption{Shared model and optimizer settings. The main experiments vary the attention kernel and data transform; the supplementary experiments also vary normalization placement. Attention arithmetic is described in Appendix~\ref{app:compute}.}
  \label{tab:hyperparams}
  \begin{tabular}{ll}
  \toprule
  hyperparameter & value \\
  \midrule
  transformer blocks & 6 \\
  model dimension $d_{\mathrm{model}}$ & 256 \\
  feed-forward width & 512 \\
  attention heads & 1 \\
  feed-forward activation & gelu \\
  normalization type & layer \\
  dropout & 0.0 \\
  causal mask & False \\
  parameters & 3,163,906 \\
  loss & Energy Distance \\
  loss exponent $p$ & 1.0 \\
  optimizer & AdamW \\
  learning rate & 0.001 \\
  weight decay & 0.0 \\
  schedule & cosine with warmup \\
  gradient clip & 1.0 \\
  skip non-finite step & True \\
  steps & 50000 \\
  samples per trial $N$ & 16384 \\
  trials per step & 64 \\
  training attention precision & bfloat16 (softmax), float32 (other kernels) \\
  evaluation precision & float32 \\
  seed & 0 \\
  \bottomrule
  \end{tabular}
\end{table}

\subsection{Compute and resource requirements}
\label{app:compute}

Every model was trained on a single NVIDIA H200 (141\,GB HBM3e). Each of the $50{,}000$ training steps in Table~\ref{tab:hyperparams} comprises one full-graph forward pass and one optimizer update over a batch of $64\times16384$ points.

Our code is written in PyTorch. During training, softmax models use scaled dot-product attention with the FlashAttention backend \citep{dao2022flashattentionfastmemoryefficientexact} under bfloat16 autocast, without materializing the $N\times N$ score matrix. The slower-growing kernels use custom attention in float32, although the transformer as a whole remains under autocast. A full score tensor would require $64\times16384^2\times4\mathrm{B}=68.7\mathrm{GB}=64\mathrm{GiB}$ per block. We instead compute query tiles capped at $2^{31}$ score elements ($8$\,GiB per float32 tile) and checkpoint each tile for recomputation during backpropagation. The custom implementation takes $11.84$\,s per training step; compiling it with \texttt{torch.compile} reduces this to $2.51$\,s. All models are evaluated in float32 without autocast.

\begin{table}[H]
  \centering
  \caption{Measured on an H200 at the trained shapes.
  Per-step times are end-to-end training steps and peak memory is the whole
  step. The sub-exponential kernels are more expensive to train since they lack a fused kernel.}
  \label{tab:compute}
  {\small\begin{tabular}{lrrr}
    \toprule
    run family & s / step & peak memory & hours for $50{,}000$ steps \\
    \midrule
    softmax (FlashAttention, bf16) & 1.87 & 69.4\,GiB & 26 \\
    sub-exponential (chunked + recompute) & 2.51 & 78.8\,GiB & 35 \\
    \bottomrule
  \end{tabular}}
\end{table}

\section{The tail-sensitivity testbed}
\label{app:tailsignal}

Section~\ref{sec:tailsensitivity} reports three statistics on the testbed of
spliced measure pairs at five sample budgets. This appendix gives the
construction of that testbed and defines the normalization the vertical axes of
Figure~\ref{fig:metricres} use.

The reference index is $a_{\mathrm{ref}}=1.5$, so the reference has a finite first moment but an infinite second moment. For the spliced candidates with $a_*>1$, the population energy distance is finite. The displayed Wasserstein distances are between finite empirical measures; when $a_*\ne a_{\mathrm{ref}}$, the differing tail exponents make the corresponding population $W_2$ infinite. The Hill gap is an auxiliary diagnostic.

\subsection{Splicing distributions to isolate the tail}
\label{app:splice}

Changing a Pareto tail index alters both the body and the tail. Splicing joins two Pareto distributions at a seam, allowing the spliced distribution to follow a reference distribution up to a specific quantile. Fix a reference index
$a_{\mathrm{ref}}$, a scale $\lambda$, a seam quantile $q_0\in(0,1)$ and the
corresponding seam value $u_0=Q_{\mathrm{ref}}(q_0)=\lambda(1-q_0)^{-1/a_{\mathrm{ref}}}$. The spliced candidate with index $a_*$ after the seam has distribution function
\begin{equation}
  G(x) \;=\;
  \begin{cases}
    0, & x < \lambda,\\[2pt]
    1 - \left(\dfrac{x}{\lambda}\right)^{-a_{\mathrm{ref}}}, & \lambda \le x \le u_0,\\[2pt]
    q_0 + (1-q_0)\left[1 - \left(\dfrac{x-b}{u_0-b}\right)^{-a_*}\right],
      & x > u_0,
  \end{cases}
  \qquad
  b \;=\; u_0\left(1-\frac{a_*}{a_{\mathrm{ref}}}\right).
  \label{eq:splice}
\end{equation}
The choice of $b$ makes the two densities agree at the seam, where $G(u_0)=q_0$, and the probability mass above the seam is $1-q_0$. The branch on $x>u_0$ in \eqref{eq:splice} is a
Pareto of index $a_*$, scale $u_0-b$, translated by $b$. Figure~\ref{fig:splice} shows this construction.

\begin{figure}[H]
  \centering
  \includegraphics[width=0.9\linewidth]{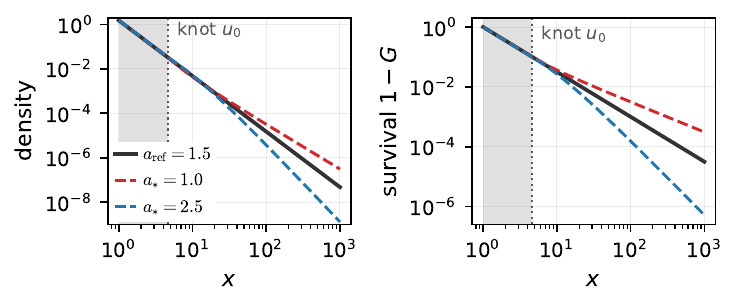}
  \caption{The splice \eqref{eq:splice} at $q_0=0.9$, reference index
  $a_{\mathrm{ref}}=\SSRefAlpha{}$ and scale $\lambda=1$, on log--log axes.
  Left: densities. Right: survival functions. The shaded region lies below the
  seam $u_0$ (dotted, denoted "knot"), where every candidate coincides with the reference; the
  candidates differ from it only to the right of the seam, over a fixed mass
  $1-q_0$.}
  \label{fig:splice}
\end{figure}

The one-dimensional sweep in Figure~\ref{fig:metricres} uses $\lambda=1$,
seam quantiles $q_0\in\{0.5,0.9,0.99\}$, and candidate tail indices
$a_*\in\{0.6,0.8,1.0,1.2,1.5,1.8,2.2,2.8,3.5\}$.
With random seed $0$, each plotted point averages $24$ independent repetitions
at $N\in\{256,1024,4096\}$, $12$ at $N=16384$, and $6$ at $N=65536$.
For the Hill gap, we apply the estimator in \eqref{eq:hill} separately to each
sample, with $k=\max\{2,\operatorname{round}(N(1-q_0))\}$ chosen from the known
seam quantile. The threshold in each sample is its $(k+1)$st largest value.

\subsection{Noise units}
\label{app:noiseunits}

Every panel of Figure~\ref{fig:metricres} plots
a shifted and scaled statistic rather than the statistic itself.

Fix one statistic, one dimension, one budget $N$, and one seam quantile $q_0$.
The sweep evaluates the statistic at each candidate index $a_j$ on the grid over
repeated independent draws, recording the mean $\mu_j$ and the
standard deviation $s_j$ between trials. Let $j^\star=\arg\min_j |a_j-a_{\mathrm{ref}}|$ be the
plotted point at the reference index $1.5$. The plotted quantity is
\begin{equation}
  z_j \;=\; \frac{\mu_j - \mu_{j^\star}}{\varsigma},
  \qquad
  \varsigma \;=\; \operatorname{median}\{\,s_j : s_j \text{ finite}\,\}.
  \label{eq:noiseunits}
\end{equation}
Across the runs in Figure~\ref{fig:metricres}, the smallest $\varsigma$ was
approximately $10^{-5}$, so no numerical floor was needed for this normalization.

\section{Pre-norm and other results}
\label{app:pre-norm}
This appendix reports supplementary pre-norm experiments and numerical summaries of the full grid. The pre-norm runs do not exhibit ensemble collapse across the tested kernels and transformations. Figures~\ref{fig:quantile} and~\ref{fig:quantileinvlin} show the distribution of energy distances over evaluation trials.

\begin{table}[H]
  \centering
  \caption{Sheared swap energy distance as a multiple of the noise
  floor at $N=16384$, under the median (top) and the mean (bottom). Bold marks ratios whose bootstrap interval contains or lies below one.
  \emph{collapse} marks runs with degenerate output ensembles, as in Figure~\ref{fig:expswap}.}
  \label{tab:energysummary}
  {\small
  \textbf{Median}\par\vspace{2pt}
  \begin{tabular}{llcccc}
  \toprule
  & & \multicolumn{2}{c}{mirrored Pareto} & \multicolumn{2}{c}{shifted Gaussian} \\
  \cmidrule(lr){3-4}\cmidrule(lr){5-6}
  norm & kernel & none & signed-log & none & signed-log \\
  \midrule
  pre-norm & Softmax ($e^{s}$) & $\mathbf{0.857}$ & $\mathbf{0.817}$ & $\mathbf{1.123}$ & $\mathbf{1.101}$ \\
  & Squared ELU ($(s{+}1)^{2}$) & $\mathbf{0.844}$ & $\mathbf{0.827}$ & $\mathbf{1.127}$ & $\mathbf{1.143}$ \\
  & Linear ELU ($s{+}1$) & $\mathbf{0.815}$ & $\mathbf{0.845}$ & $\mathbf{1.123}$ & $\mathbf{1.074}$ \\
  & Log ELU ($\log(1{+}s){+}1$) & $\mathbf{0.795}$ & $\mathbf{0.806}$ & $\mathbf{1.057}$ & $\mathbf{1.097}$ \\
  \addlinespace
  post-norm, res.\ scale & Softmax ($e^{s}$) & \emph{collapse} & \emph{collapse} & $\mathbf{1.119}$ & $\mathbf{1.097}$ \\
  & Squared ELU ($(s{+}1)^{2}$) & $\mathbf{1.174}$ & $\mathbf{0.863}$ & $\mathbf{1.088}$ & $\mathbf{1.130}$ \\
  & Linear ELU ($s{+}1$) & $\mathbf{0.980}$ & $\mathbf{0.833}$ & $\mathbf{1.082}$ & $\mathbf{1.092}$ \\
  & Log ELU ($\log(1{+}s){+}1$) & $\mathbf{1.006}$ & $\mathbf{0.767}$ & $\mathbf{1.051}$ & $\mathbf{1.147}$ \\
  \bottomrule
  \end{tabular}\par\vspace{8pt}
  \textbf{Mean}\par\vspace{2pt}
  \begin{tabular}{llcccc}
  \toprule
  & & \multicolumn{2}{c}{mirrored Pareto} & \multicolumn{2}{c}{shifted Gaussian} \\
  \cmidrule(lr){3-4}\cmidrule(lr){5-6}
  norm & kernel & none & signed-log & none & signed-log \\
  \midrule
  pre-norm & Softmax ($e^{s}$) & $\mathbf{0.833}$ & $\mathbf{0.779}$ & $\mathbf{1.148}$ & $\mathbf{1.142}$ \\
  & Squared ELU ($(s{+}1)^{2}$) & $\mathbf{0.849}$ & $\mathbf{0.770}$ & $\mathbf{1.129}$ & $\mathbf{1.136}$ \\
  & Linear ELU ($s{+}1$) & $\mathbf{0.828}$ & $\mathbf{0.777}$ & $\mathbf{1.142}$ & $\mathbf{1.164}$ \\
  & Log ELU ($\log(1{+}s){+}1$) & $\mathbf{0.847}$ & $\mathbf{0.749}$ & $\mathbf{1.127}$ & $1.190$ \\
  \addlinespace
  post-norm, res.\ scale & Softmax ($e^{s}$) & \emph{collapse} & \emph{collapse} & $1.181$ & $1.185$ \\
  & Squared ELU ($(s{+}1)^{2}$) & $1.707$ & $\mathbf{0.805}$ & $\mathbf{1.153}$ & $1.192$ \\
  & Linear ELU ($s{+}1$) & $\mathbf{1.317}$ & $\mathbf{0.804}$ & $\mathbf{1.171}$ & $1.189$ \\
  & Log ELU ($\log(1{+}s){+}1$) & $1.474$ & $\mathbf{0.750}$ & $1.189$ & $1.215$ \\
  \bottomrule
  \end{tabular}}
\end{table}

\begin{figure}[H]
  \centering
  \includegraphics[width=\linewidth]{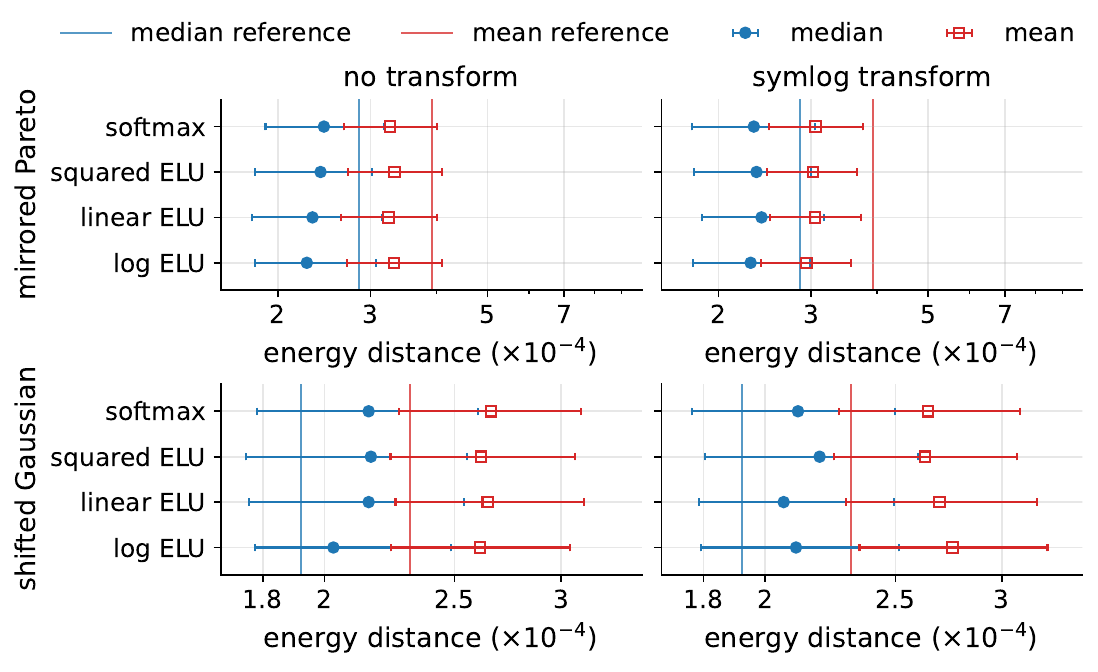}
  \caption{The equivalent of Figure~\ref{fig:expswap} for pre-norm.}
  \label{fig:energysummary}
\end{figure}

\begin{table}[H]
  \centering
  \caption{The counterpart of Table~\ref{tab:energysummary} for the matrix inverse operator on mirrored Pareto marginals.}
  \label{tab:energysummaryinvlin}
  {\small
  \textbf{Median}\par\vspace{2pt}
  \begin{tabular}{llcc}
  \toprule
  & & \multicolumn{2}{c}{mirrored Pareto} \\
  \cmidrule(lr){3-4}
  norm & kernel & none & signed-log \\
  \midrule
  pre-norm & Softmax ($e^{s}$) & $\mathbf{1.085}$ & $\mathbf{1.053}$ \\
  & Squared ELU ($(s{+}1)^{2}$) & $\mathbf{1.145}$ & $\mathbf{1.133}$ \\
  & Linear ELU ($s{+}1$) & $\mathbf{1.143}$ & $\mathbf{1.092}$ \\
  & Log ELU ($\log(1{+}s){+}1$) & $\mathbf{1.181}$ & $\mathbf{1.164}$ \\
  \addlinespace
  post-norm, res.\ scale & Softmax ($e^{s}$) & \emph{collapse} & $\mathbf{1.120}$ \\
  & Squared ELU ($(s{+}1)^{2}$) & $3.423$ & $\mathbf{1.167}$ \\
  & Linear ELU ($s{+}1$) & $\mathbf{1.295}$ & $\mathbf{1.149}$ \\
  & Log ELU ($\log(1{+}s){+}1$) & $\mathbf{1.354}$ & $\mathbf{1.135}$ \\
  \bottomrule
  \end{tabular}\par\vspace{8pt}
  \textbf{Mean}\par\vspace{2pt}
  \begin{tabular}{llcc}
  \toprule
  & & \multicolumn{2}{c}{mirrored Pareto} \\
  \cmidrule(lr){3-4}
  norm & kernel & none & signed-log \\
  \midrule
  pre-norm & Softmax ($e^{s}$) & $\mathbf{0.869}$ & $\mathbf{0.806}$ \\
  & Squared ELU ($(s{+}1)^{2}$) & $\mathbf{0.933}$ & $\mathbf{0.797}$ \\
  & Linear ELU ($s{+}1$) & $\mathbf{0.816}$ & $\mathbf{0.807}$ \\
  & Log ELU ($\log(1{+}s){+}1$) & $\mathbf{0.889}$ & $\mathbf{0.868}$ \\
  \addlinespace
  post-norm, res.\ scale & Softmax ($e^{s}$) & \emph{collapse} & $\mathbf{0.806}$ \\
  & Squared ELU ($(s{+}1)^{2}$) & $3.718$ & $\mathbf{0.849}$ \\
  & Linear ELU ($s{+}1$) & $1.427$ & $\mathbf{0.866}$ \\
  & Log ELU ($\log(1{+}s){+}1$) & $1.466$ & $\mathbf{0.821}$ \\
  \bottomrule
  \end{tabular}}
\end{table}

\begin{figure}[H]
  \centering
  \includegraphics[width=\linewidth]{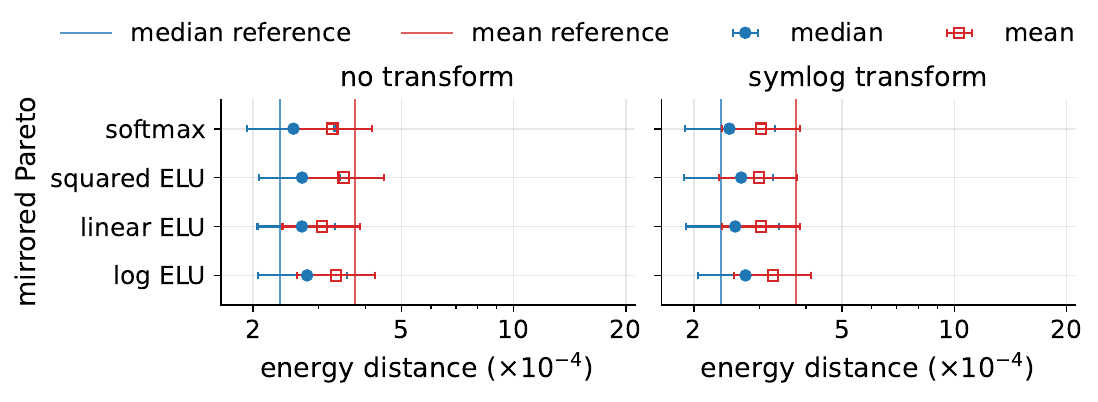}
  \caption{The equivalent of Figure~\ref{fig:expinvlin} for pre-norm.}
  \label{fig:energysummaryinvlin}
\end{figure}

\begin{figure}[H]
  \centering
  \includegraphics[width=0.86\linewidth]{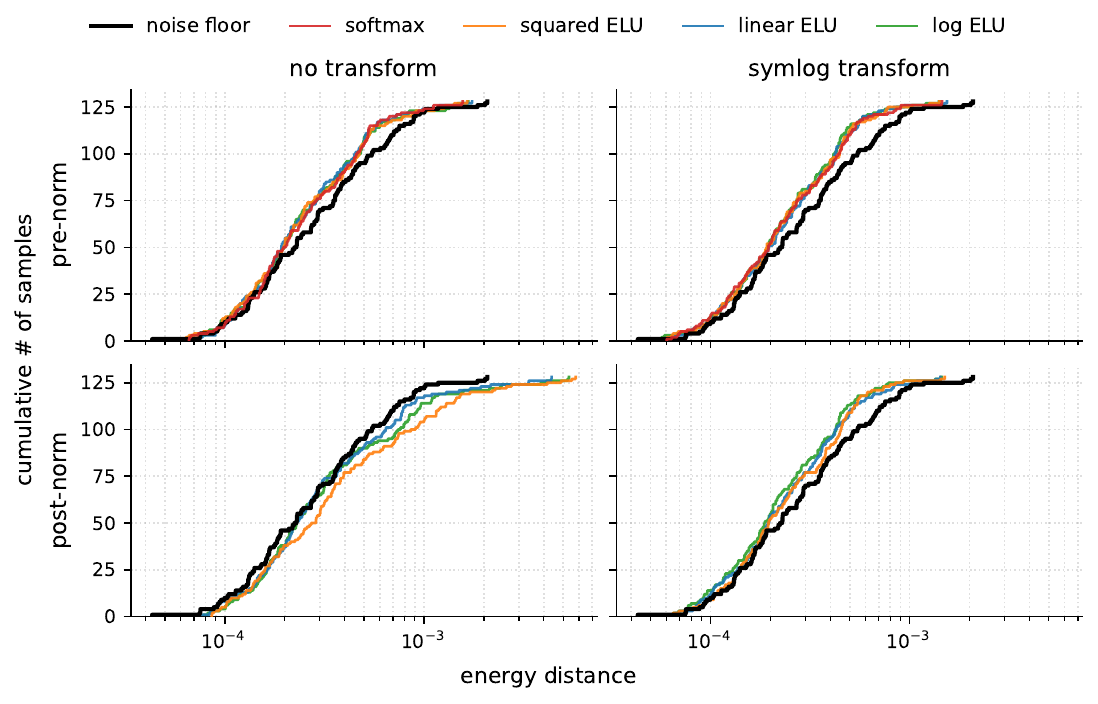}
  \caption{The empirical CDF of the per-trial energy distance for the sheared
  swap operator on mirrored Pareto marginals at $N=16384$, organized by
  normalization placement (rows) and symlog transform (columns). The horizontal
  axis is the energy distance on a log scale and the vertical axis is the
  cumulative count of each curve's $128$ trials, so a curve further left has lower energy distances. The black curve is the noise floor. Softmax is
  absent from
  both post-norm panels because it collapses there.}
  \label{fig:quantile}
\end{figure}

\begin{figure}[H]
  \centering
  \includegraphics[width=0.86\linewidth]{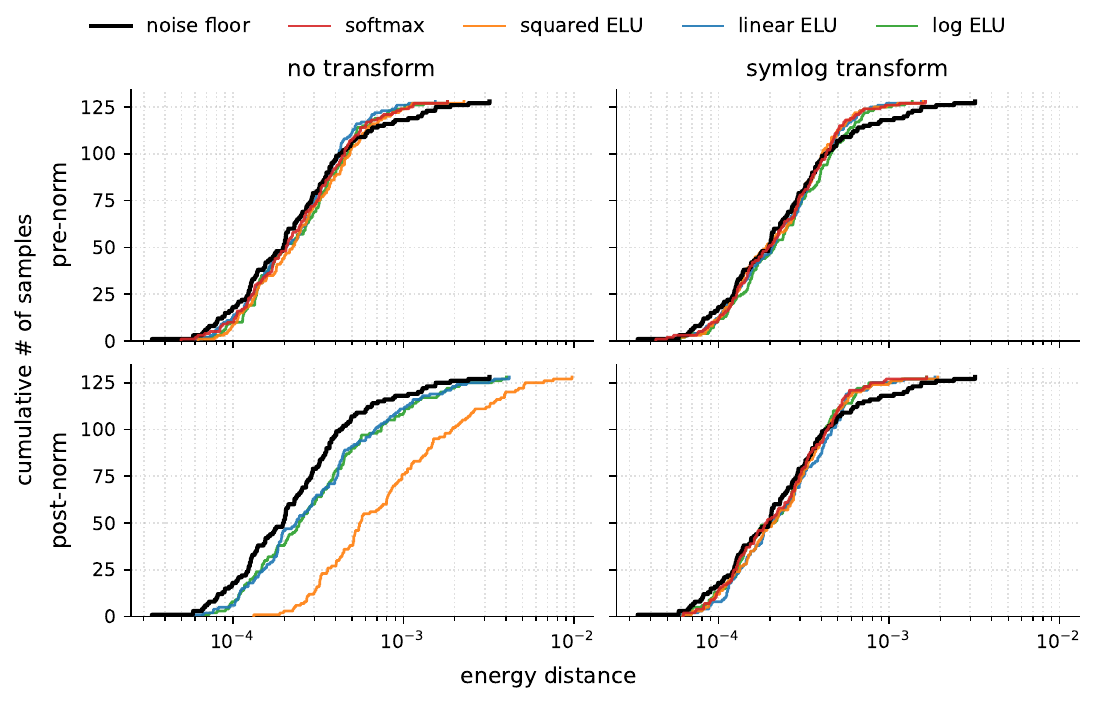}
  \caption{The equivalent of Figure~\ref{fig:quantile} for the matrix inverse
  operator. Here only the post-norm softmax curve without the symlog transform is
  omitted for collapse.}
  \label{fig:quantileinvlin}
\end{figure}

\section{Reading the ablation values}
\label{app:reading}

\subsection{Why a reported score can fall below the sampling reference}
\label{app:ratios}

Every reported number in Section~\ref{sec:ablation} is compared with its matched i.i.d. sampling reference
\eqref{eq:noise-floor}: the energy distance between two independent draws of the
target at the same sample budget. The figures plot the energy distance itself and draw the reference as a vertical line; the tables report the quotient of the two. A score below the reference is not evidence that a model output is better than the target. The following identities explain how a deterministic ensemble can score below this reference.

Write $\delta=\mathbb{E}\|Y-Y'\|$ for independent draws of the target $P$, and
write $\mu_U$ for the empirical measure of a size-$N$ point set $U$. The V-statistic energy distance is the population energy distance between empirical measures,
\begin{equation}
  \widehat{\mathcal{E}^{2}}(U,V) \;=\; \mathcal{E}_1^{2}(\mu_U,\mu_V).
  \label{eq:plugin}
\end{equation}

\paragraph{The i.i.d. sampling reference.}
Let $P_N$ and $P_N'$ be independent size-$N$ i.i.d. samples of $P$. The cross
sum in \eqref{eq:energy} pairs points from two independent samples, so each of
its $N^{2}$ summands has mean $\delta$. A within-sample sum instead pairs a
sample with itself: $N$ of its $N^{2}$ summands are the distance between a point and itself
$\|Y_i-Y_i\|=0$ and the remaining $N^{2}-N$ have mean $\delta$, so the within-sample average has
mean $(1-1/N)\,\delta$. Both within-sample sums carry a minus sign in the energy distance, hence
\begin{equation}
  \mathbb{E}\big[\widehat{\mathcal{E}^{2}}(P_N,P_N')\big]
  \;=\; 2\delta - 2\Bigl(1-\tfrac{1}{N}\Bigr)\delta
  \;=\; \frac{2\delta}{N}.
  \label{eq:floorexp}
\end{equation}

\paragraph{A deterministic output.}
A deterministic model emits a fixed point set $R$ given its input. Let $\mu_R$ be the empirical measure of $R$. Then 
\begin{equation}
  \mathbb{E}\big[\widehat{\mathcal{E}^{2}}(R, P_N)\big]
  \;=\; \mathcal{E}_1^{2}(\mu_R, P) \;+\; \delta/N ,
  \label{eq:fixedset}
\end{equation}
where the average pairwise distance within $R$ equals $\mathbb{E}\|U-U^{\prime}\|$ for independent $U,U^{\prime}\sim\mu_R$. Thus, when $\mathcal{E}_1^{2}(\mu_R, P)$ is below $\delta/N$ the left hand side of \eqref{eq:fixedset} can be lower than the i.i.d. sampling reference, causing the reported score to fall below the reference line and the reported ratio to fall below $1$.

\subsection{The bootstrap interval}
\label{app:bootstrap}

Every reported number in Section~\ref{sec:ablation} summarizes $n=128$ evaluation trials. Trial $i$ draws its
own parameter $\omega^{(i)}$, then produces an input ensemble $X^{(i)} \sim \mu_I(\omega^{(i)})^{\otimes N}$ and a
target ensemble $Y_1^{(i)} \sim \mu_O(\omega^{(i)})^{\otimes N}$ each of $N=16384$ points. The per-trial energy distance is $d_i = \mathcal{E}^{2}(\widetilde{\Psi}_{\theta}(X^{(i)}), Y_1^{(i)})$. To compute a noise floor, the process is repeated for another $n$ trials, indexed by $j$, where the energy distance $\hat{d}_j = \mathcal{E}^{2}(Y_2^{(j)}, Y_3^{(j)})$ is computed from two independent target ensembles $Y_2, Y_3 \sim \mu_O(\omega^{(j)})^{\otimes N}$. The reported ratio is then computed,
\begin{equation}
  \widehat{r} \;=\; \frac{S\big(d_1,\dots,d_n\big)}{S\big(\hat{d}_1,\dots,\hat{d}_n\big)},
  \qquad S \in \{\operatorname{median}, \operatorname{mean}\},
  \label{eq:ratiostat}
\end{equation}
so the median ratio divides the median candidate distance by the median reference distance, while the mean ratio uses the corresponding means.

For each operator and marginal family, we use the same sampling reference for all models.

We also compute a nonparametric bootstrap over trials. For $b=1,\dots,B$ with
$B=2000$, we draw $i^{(b)}_1,\dots,i^{(b)}_n$ and $j^{(b)}_1,\dots,j^{(b)}_n$
uniformly with replacement from
$\{1,\dots,n\}$, form
\begin{equation}
  r_b \;=\;
  \frac{S\big(d_{i^{(b)}_1},\dots,d_{i^{(b)}_n}\big)}
       {S\big(\hat{d}_{j^{(b)}_1},\dots,\hat{d}_{j^{(b)}_n}\big)},
  \label{eq:bootdraw}
\end{equation}
and report the empirical $2.5\%$ and $97.5\%$ quantiles of $r_1,\dots,r_B$ as the
$95\%$ interval.

The figures show this same interval on the scale of the energy distance: both endpoints are multiplied through by the reference summary $S\big(\hat{d}_1,\dots,\hat{d}_n\big)$, which is the vertical line the bar is read against. An interval that contains or lies below the reference corresponds to a bold entry in Tables~\ref{tab:energysummary} and~\ref{tab:energysummaryinvlin}.

These intervals describe variability across evaluation task instances at fixed trained weights. For heavy-tailed mean ratios, they are descriptive bootstrap ranges; their nominal $95\%$ coverage is not established. They do not quantify variation across training runs (Section~\ref{sec:limitations}).

\section{Ensemble collapse}
\label{app:collapse}

Figure~\ref{fig:collapse} displays the output of a collapsed model, and
Figure~\ref{fig:curves} displays when collapse occurs during training.

\begin{figure}[H]
  \centering
  \includegraphics[width=\linewidth]{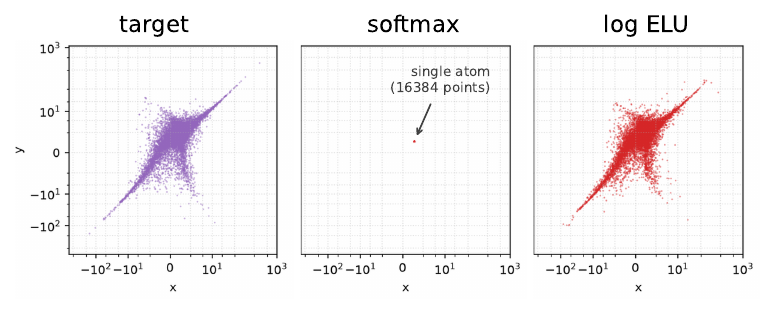}
  \caption{Model outputs for the sheared swap operator on mirrored Pareto
  marginals, using post-norm and no symlog transform, for a single evaluation
  instance of $16384$ points. Left: the target measure. Center: the softmax
  output, concentrated near the single location marked by the
  arrow. Right: the log-ELU output. All three panels share symmetric-log axes on
  both coordinates.}
  \label{fig:collapse}
\end{figure}

\subsection{Training curves}
\label{app:curves}

\begin{figure}[H]
  \centering
  \includegraphics[width=\linewidth]{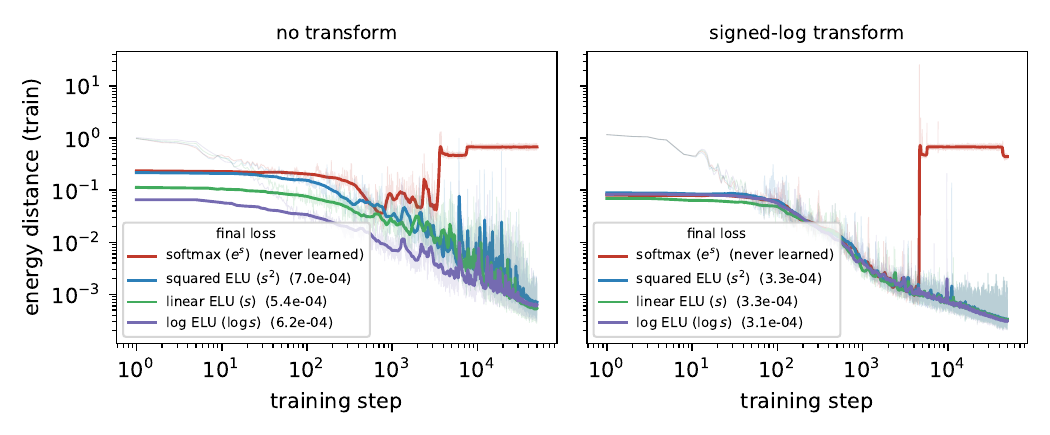}
  \caption{Training curves for the four attention kernels of
  Table~\ref{tab:family} on the sheared swap operator with mirrored Pareto
  marginals, using post-norm, with and without the symlog transform. Faint
  traces are the per-step loss on a fresh batch and heavy traces are a rolling
  median. Softmax (red) descends alongside the other three kernels, reaching
  $1.8\times10^{-2}$ at step $733$ without the transform and $6.3\times10^{-4}$
  at step $3931$ with it, then diverges abruptly at step ${\approx}3500$ and
  ${\approx}4500$ to the collapsed value it holds for the remaining $45{,}000$
  steps. The other three kernels continue to descend. Collapse is therefore not
  a consequence of stopping training early.}
  \label{fig:curves}
\end{figure}

\section{Hill-based tail diagnostics}
\label{app:tailstat}

We complement the aggregate energy distances with a diagnostic of the upper-tail samples. For each model, we compare its discrepancy from a target Hill estimate with the discrepancy between two independent target Hill estimates.

\subsection{The statistic}
\label{app:hillstat}

Fix an attention kernel, a normalization placement, a transform, an
operator, a marginal family, a coordinate $c\in\{0,1\}$ and an evaluation size
$N$. Write $Z_{(1)}\ge\cdots\ge Z_{(N)}$ for the descending order statistics of
coordinate $c$ of a size-$N$ ensemble $Z$. The Hill estimate of that coordinate's
tail index is
\begin{equation}
  \widehat{a}_{c}(Z)
    \;=\;\Bigl(\frac{1}{k}\sum_{\ell=1}^{k}
        \log\frac{Z_{(\ell)}}{Z_{(k+1)}}\Bigr)^{-1},
  \qquad k=\lfloor \phi N\rfloor ,
  \label{eq:hill}
\end{equation}
where $\phi=0.1$ at $N=16384$ and $\phi=0.01$ at $N=65536$. The two settings use $k=1638$ and $k=655$, respectively. We apply a numerical floor of $10^{-16}$ to the sample values used in the logarithmic ratios; all Hill estimates in the reported evaluations were finite. We use this diagnostic only for mirrored Pareto marginals, whose tail-index parameter is $a$ in \eqref{eq:Pareto}.

We chose these tail fractions from Figure~\ref{fig:metricres}, where the energy
distance responds to tail differences at $q_0=0.9$ for $N=16384$ and at
$q_0=0.99$ for $N=65536$.

In trial $i$ of $n=128$, we draw a parameter $\omega^{(i)}$, an input ensemble $X^{(i)}\sim\mu_I(\omega^{(i)})^{\otimes N}$ and an independent target ensemble $Y_1^{(i)}\sim\mu_O(\omega^{(i)})^{\otimes N}$. The output transformation is $T_i=M(\omega_T^{(i)})P$ for the sheared swap and $T_i=B(\omega_T^{(i)})^{-1}$ for the matrix inverse. We apply $T_i^{-1}$ to each output point before computing the Hill estimate, recovering the latent coordinate order. For the row-sample matrices used here, write
\[
  \mathcal{H}_{i,c}(Y)=\widehat a_c\bigl(Y(T_i^{-1})^\top\bigr).
\]
Axis $0$ has the lighter tail ($a\in[2,3]$) and axis $1$ the heavier tail ($a\in[1+\epsilon,2]$) throughout this appendix. The candidate discrepancy is

\begin{equation}
  e_i \;=\; \bigl|\mathcal{H}_{i,c}\bigl(\widetilde{\Psi}_{\theta}(X^{(i)})\bigr)
           -\mathcal{H}_{i,c}\bigl(Y_1^{(i)}\bigr)\bigr|.
  \label{eq:hillerr}
\end{equation}
An independent pool of $n$ trials draws parameters $\omega^{(j)}$ and two independent target ensembles $Y_2^{(j)},Y_3^{(j)}\sim\mu_O(\omega^{(j)})^{\otimes N}$. Using each trial's output transformation $T_j$, the reference discrepancy is

\begin{equation}
  \hat{e}_j \;=\; \bigl|\mathcal{H}_{j,c}\bigl(Y_2^{(j)}\bigr)
           -\mathcal{H}_{j,c}\bigl(Y_3^{(j)}\bigr)\bigr|.
  \label{eq:hillerrref}
\end{equation}
We compare $e$ and $\hat{e}$ using the U-statistic estimator of the one-dimensional energy distance with $p=1$:
\begin{equation}
  E(e,\hat{e})
    \;=\;\frac{2}{n^2}\sum_{i,j}\bigl|e_i-\hat{e}_j\bigr|
    \;-\;\frac{1}{n(n-1)}\sum_{i\neq i'}\bigl|e_i-e_{i'}\bigr|
    \;-\;\frac{1}{n(n-1)}\sum_{j\neq j'}\bigl|\hat{e}_j-\hat{e}_{j'}\bigr| .
  \label{eq:hillenergystat}
\end{equation}
For independent i.i.d. trial pools with finite first moments of $e$ and $\hat e$, \eqref{eq:hillenergystat} is unbiased for the energy distance between their distributions. If these distributions agree, its expectation is zero at finite $n$. We use it as a descriptive measure of distributional discrepancy: it does not determine which group has smaller Hill errors, and its sign alone is not a significance test. Finite-sample estimates may be negative. Appendix~\ref{app:hillresults} reports the results.
\subsection{Results}
\label{app:hillresults}

\begin{table}[H]
  \centering
  \caption{$E(e, \hat{e})$ from \eqref{eq:hillenergystat} at $N=16384$ for the sheared swap operator. The statistic has expectation $0$ when the distributions of candidate and reference Hill errors agree. Axis $0$
  carries the lighter tail ($a\in[2,3]$) and axis $1$ the heavier ($a\in[1+\epsilon,2]$).
  The two softmax post-norm entries are the collapsed runs of
  Table~\ref{tab:energysummary}.}
  \label{tab:hillenergy}
  {\small
  \textbf{Axis 0} (lighter tail)\par\vspace{2pt}
  \begin{tabular}{llcc}
  \toprule
  & & \multicolumn{2}{c}{mirrored Pareto} \\
  \cmidrule(lr){3-4}
  norm & kernel & none & signed-log \\
  \midrule
  pre-norm & Softmax ($e^{s}$) & $7.6{\times}10^{-3}$ & $2.5{\times}10^{-3}$ \\
  & Squared ELU ($(s{+}1)^{2}$) & $5.6{\times}10^{-3}$ & $1.5{\times}10^{-3}$ \\
  & Linear ELU ($s{+}1$) & $2.1{\times}10^{-3}$ & $6.7{\times}10^{-4}$ \\
  & Log ELU ($\log(1{+}s){+}1$) & $8.1{\times}10^{-3}$ & $9.3{\times}10^{-4}$ \\
  \addlinespace
  post-norm, res.\ scale & Softmax ($e^{s}$) & $4.3{\times}10^{8}$ & $1.8{\times}10^{5}$ \\
  & Squared ELU ($(s{+}1)^{2}$) & $3.5{\times}10^{-2}$ & $1.6{\times}10^{-3}$ \\
  & Linear ELU ($s{+}1$) & $2.2{\times}10^{-2}$ & $1.8{\times}10^{-3}$ \\
  & Log ELU ($\log(1{+}s){+}1$) & $3.5{\times}10^{-2}$ & $-5.2{\times}10^{-4}$ \\
  \bottomrule
  \end{tabular}\par\vspace{8pt}
  \textbf{Axis 1} (heavier tail)\par\vspace{2pt}
  \begin{tabular}{llcc}
  \toprule
  & & \multicolumn{2}{c}{mirrored Pareto} \\
  \cmidrule(lr){3-4}
  norm & kernel & none & signed-log \\
  \midrule
  pre-norm & Softmax ($e^{s}$) & $-7.2{\times}10^{-5}$ & $-1.2{\times}10^{-4}$ \\
  & Squared ELU ($(s{+}1)^{2}$) & $-1.5{\times}10^{-4}$ & $7.2{\times}10^{-5}$ \\
  & Linear ELU ($s{+}1$) & $3.8{\times}10^{-4}$ & $5.2{\times}10^{-5}$ \\
  & Log ELU ($\log(1{+}s){+}1$) & $3.5{\times}10^{-6}$ & $5.5{\times}10^{-5}$ \\
  \addlinespace
  post-norm, res.\ scale & Softmax ($e^{s}$) & $1.6{\times}10^{9}$ & $7.0{\times}10^{3}$ \\
  & Squared ELU ($(s{+}1)^{2}$) & $-2.4{\times}10^{-4}$ & $8.7{\times}10^{-5}$ \\
  & Linear ELU ($s{+}1$) & $-1.1{\times}10^{-4}$ & $2.1{\times}10^{-4}$ \\
  & Log ELU ($\log(1{+}s){+}1$) & $-2.6{\times}10^{-4}$ & $2.3{\times}10^{-5}$ \\
  \bottomrule
  \end{tabular}}
\end{table}

\begin{table}[H]
  \centering
  \caption{The equivalent of Table~\ref{tab:hillenergy} for the matrix inverse
  operator.}
  \label{tab:hillenergyinvlin}
  {\small
  \textbf{Axis 0} (lighter tail)\par\vspace{2pt}
  \begin{tabular}{llcc}
  \toprule
  & & \multicolumn{2}{c}{mirrored Pareto} \\
  \cmidrule(lr){3-4}
  norm & kernel & none & signed-log \\
  \midrule
  pre-norm & Softmax ($e^{s}$) & $1.8{\times}10^{-2}$ & $5.1{\times}10^{-4}$ \\
  & Squared ELU ($(s{+}1)^{2}$) & $8.8{\times}10^{-3}$ & $4.7{\times}10^{-5}$ \\
  & Linear ELU ($s{+}1$) & $4.5{\times}10^{-3}$ & $3.5{\times}10^{-3}$ \\
  & Log ELU ($\log(1{+}s){+}1$) & $8.1{\times}10^{-3}$ & $-3.4{\times}10^{-4}$ \\
  \addlinespace
  post-norm, res.\ scale & Softmax ($e^{s}$) & $2.1{\times}10^{8}$ & $-2.5{\times}10^{-4}$ \\
  & Squared ELU ($(s{+}1)^{2}$) & $1.4{\times}10^{-1}$ & $2.4{\times}10^{-3}$ \\
  & Linear ELU ($s{+}1$) & $3.2{\times}10^{-2}$ & $1.1{\times}10^{-3}$ \\
  & Log ELU ($\log(1{+}s){+}1$) & $3.6{\times}10^{-2}$ & $1.7{\times}10^{-4}$ \\
  \bottomrule
  \end{tabular}\par\vspace{8pt}
  \textbf{Axis 1} (heavier tail)\par\vspace{2pt}
  \begin{tabular}{llcc}
  \toprule
  & & \multicolumn{2}{c}{mirrored Pareto} \\
  \cmidrule(lr){3-4}
  norm & kernel & none & signed-log \\
  \midrule
  pre-norm & Softmax ($e^{s}$) & $3.3{\times}10^{-4}$ & $1.3{\times}10^{-4}$ \\
  & Squared ELU ($(s{+}1)^{2}$) & $-1.0{\times}10^{-4}$ & $-1.1{\times}10^{-4}$ \\
  & Linear ELU ($s{+}1$) & $-8.3{\times}10^{-6}$ & $2.2{\times}10^{-4}$ \\
  & Log ELU ($\log(1{+}s){+}1$) & $1.2{\times}10^{-4}$ & $-1.8{\times}10^{-4}$ \\
  \addlinespace
  post-norm, res.\ scale & Softmax ($e^{s}$) & $8.9{\times}10^{7}$ & $-1.9{\times}10^{-4}$ \\
  & Squared ELU ($(s{+}1)^{2}$) & $5.9{\times}10^{-3}$ & $-2.3{\times}10^{-4}$ \\
  & Linear ELU ($s{+}1$) & $5.9{\times}10^{-4}$ & $-1.1{\times}10^{-4}$ \\
  & Log ELU ($\log(1{+}s){+}1$) & $-5.7{\times}10^{-5}$ & $2.5{\times}10^{-4}$ \\
  \bottomrule
  \end{tabular}}
\end{table}

\begin{figure}[H]
  \centering
  \includegraphics[width=\linewidth]{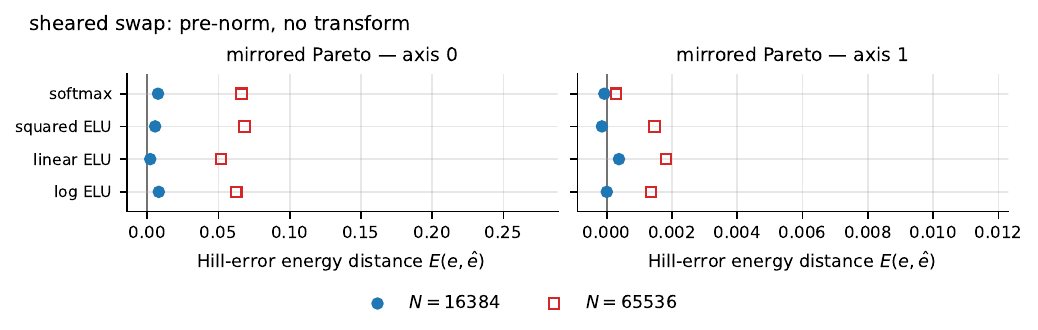}\par\vspace{6pt}
  \includegraphics[width=\linewidth]{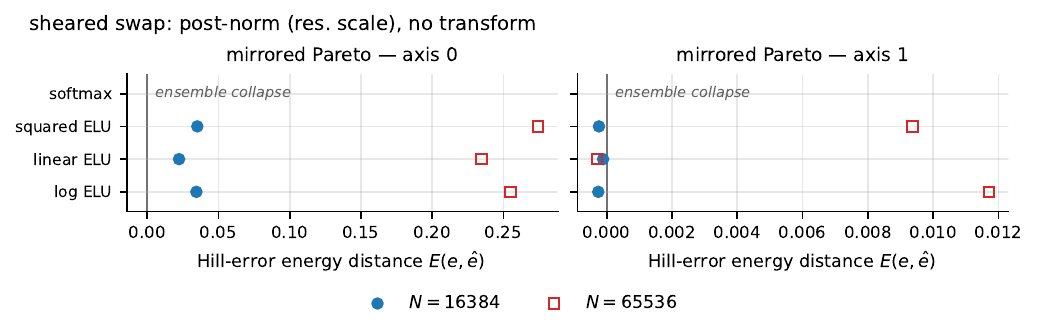}
  \caption{$E(e, \hat{e})$ for the sheared swap operator without
  the symlog transform, using pre-norm (top) and post-norm (bottom). The vertical line marks $0$. Larger population values indicate a greater difference between the two distributions of Hill errors; the finite-trial estimate can be negative. The two axes are scaled separately. Appendix~\ref{app:hillstat} defines the sample sizes and tail fractions. Runs that exhibited ensemble collapse in Section~\ref{sec:ablation} are again labeled "ensemble collapse".}
  \label{fig:hillenergy}
\end{figure}

\begin{figure}[H]
  \centering
  \includegraphics[width=\linewidth]{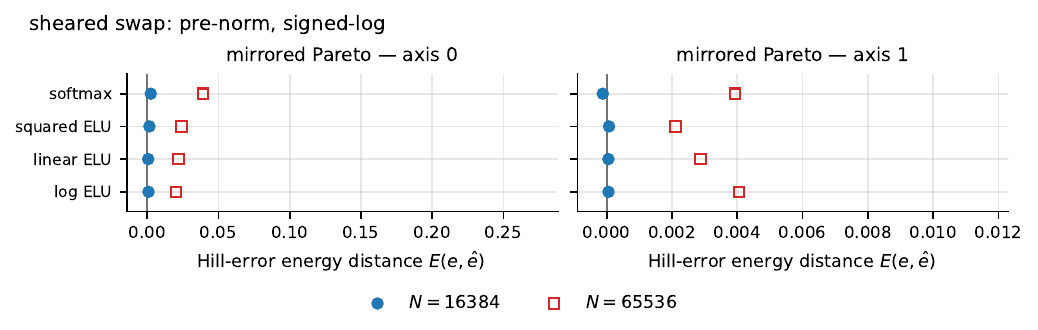}\par\vspace{6pt}
  \includegraphics[width=\linewidth]{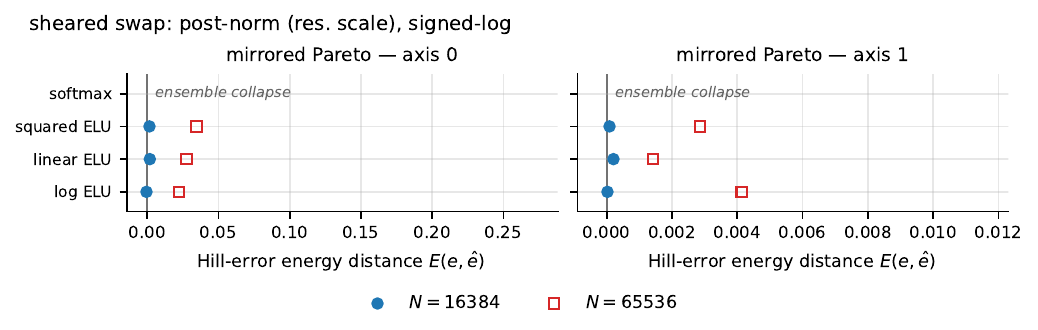}
  \caption{The equivalent of Figure~\ref{fig:hillenergy} with the symlog
  transform.}
  \label{fig:hillenergylog}
\end{figure}

\begin{figure}[H]
  \centering
  \includegraphics[width=\linewidth]{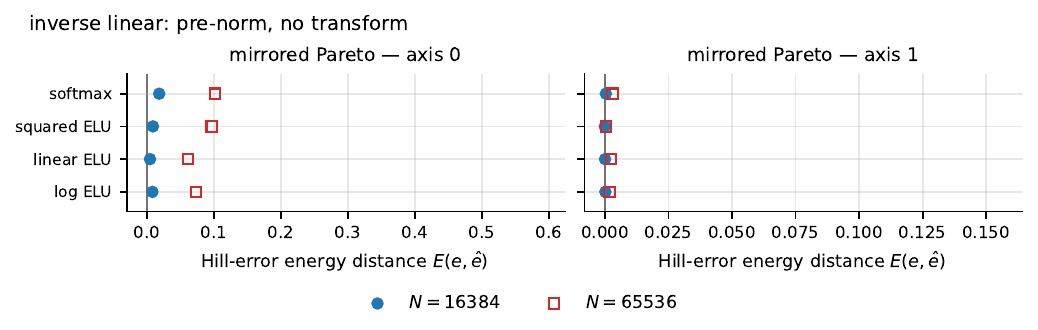}\par\vspace{6pt}
  \includegraphics[width=\linewidth]{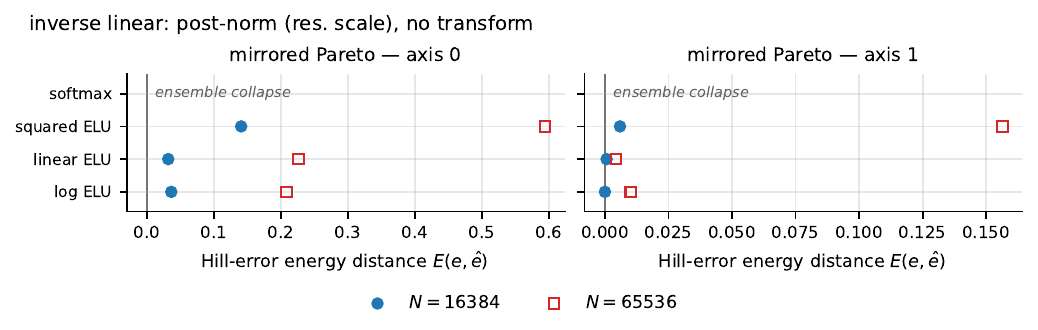}
  \caption{The equivalent of Figure~\ref{fig:hillenergy} for the matrix inverse
  operator.}
  \label{fig:hillenergyinvlin}
\end{figure}

\begin{figure}[H]
  \centering
  \includegraphics[width=\linewidth]{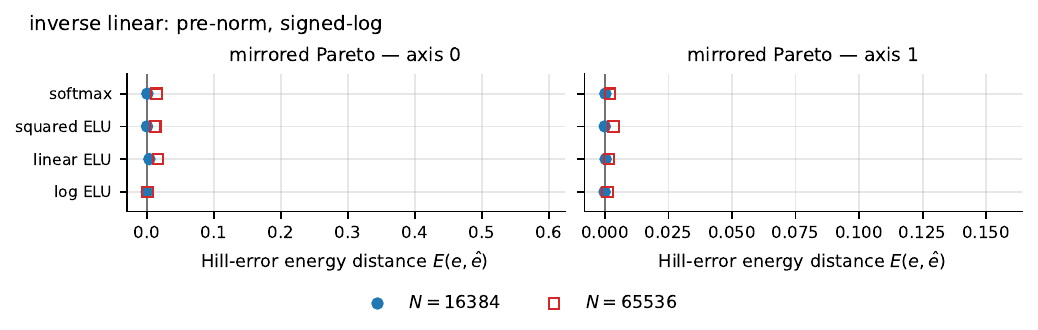}\par\vspace{6pt}
  \includegraphics[width=\linewidth]{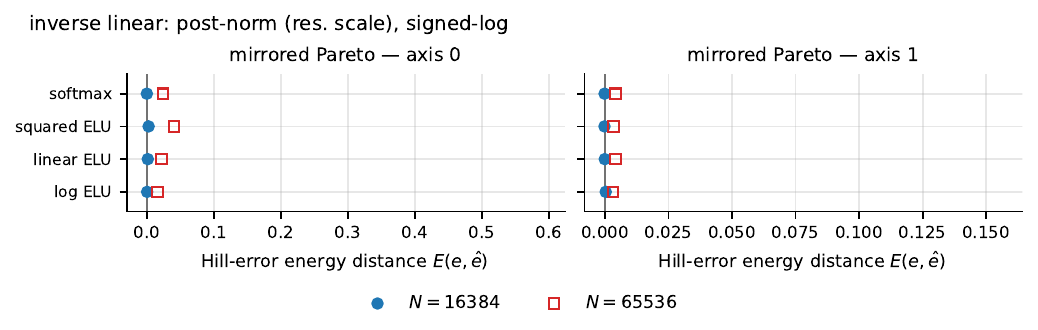}
  \caption{The equivalent of Figure~\ref{fig:hillenergy} for the matrix
  inverse operator and symlog transform.}
  \label{fig:hillenergyinvlinlog}
\end{figure}

\end{document}